\documentclass[lettersize,journal]{IEEEtran}
\usepackage{amsmath,amsfonts}
\usepackage{algorithmicx}
\usepackage{algorithm}
\usepackage{array}
\usepackage[caption=false,font=normalsize,labelfont=sf,textfont=sf]{subfig}
\usepackage{textcomp}
\usepackage{stfloats}
\usepackage{url}
\usepackage{verbatim}
\usepackage{graphicx}

\usepackage{cite}
\usepackage{xcolor}
\usepackage[colorlinks,bookmarksopen,bookmarksnumbered,citecolor=red,urlcolor=red]{hyperref}
\usepackage{amsthm}
\usepackage{times}
\usepackage{bm}
\usepackage{multirow}
\usepackage{galois}
\usepackage{mathrsfs}
\usepackage{amssymb}
\usepackage{algpseudocode}

\usepackage{makecell} 

\usepackage{threeparttable}

\newtheorem{definition}{Definition}

\newtheorem{problem}[definition]{Problem}
\newtheorem{remark}[definition]{Remark}
\newtheorem{theorem}[definition]{Theorem}
\newtheorem{lemma}[definition]{Lemma}
\newtheorem{corollary}[definition]{Corollary}

\floatname{algorithm}{Algorithm}

\begin{document}

\title{Efficient Distance-Optimal Tethered Path Planning in Planar Environments: The Workspace Convexity}

\author{Tong Yang, Rong Xiong, and Yue Wang$^*$
\thanks{Tong Yang, Rong Xiong, and Yue Wang are with the State Key Laboratory of Industrial Control and Technology, Zhejiang University, P.R. China. Yue Wang is the corresponding author {\tt\small ywang24@zju.edu.cn}.}
}
\maketitle

\begin{abstract}
The main contribution of this paper is the proof of the convexity of the omni-directional tethered robot workspace (namely, the set of all tether-length-admissible robot configurations), as well as a set of distance-optimal tethered path planning algorithms that leverage the workspace convexity. 
The workspace is proven to be topologically a simply-connected subset and geometrically a convex subset of the set of all configurations. 
As a direct result, the tether-length-admissible optimal path between two configurations is proven exactly the untethered collision-free locally shortest path in the homotopy specified by the concatenation of the tether curve of the given configurations, which can be simply constructed by performing an untethered path shortening process in the 2D environment instead of a path searching process in the pre-calculated workspace. 

The convexity is an intrinsic property to the tethered robot kinematics, thus has universal impacts on all high-level distance-optimal tethered path planning tasks: The most time-consuming workspace pre-calculation (WP) process is replaced with a goal configuration pre-calculation (GCP) process, and the homotopy-aware path searching process is replaced with untethered path shortening processes. 
Motivated by the workspace convexity, efficient algorithms to solve the following problems are naturally proposed: (a) The optimal tethered reconfiguration (TR) planning problem is solved by a locally untethered path shortening (UPS) process, (b) The classic optimal tethered path (TP) planning problem (from a starting configuration to a goal location whereby the target tether state is not assigned) is solved by a GCP process and $n$ UPS processes, where $n$ is the number of tether-length-admissible configurations that visit the goal location, (c) The optimal tethered motion to visit a sequence of multiple goal locations, referred to as the optimal tethered multi-goal visiting (TMV) task, is solved by a GCP process, $\sum_{i = 1}^{N-1} n_in_{i+1}$ UPS processes, and a numerical dynamic programming (DP) process, where $N$ is the number of goals and $n_i$ is the number of configurations to visit the $i$-th goal, and 
(d) The optimal tethered travelling salesman problem (TTSP) is solved by a GCP process, $(\sum_{i=1}^{N} n_i)^2 - \sum_{i=1}^Nn_i^2$ UPS processes, and a numerical generalised travelling salesman problem (GTSP) process. 
Challenging simulation scenarios are presented to validate the computational advantage of the proposed contribution. 
\end{abstract}

\begin{IEEEkeywords}
Optimal Tethered Path Planning, Multi-Goal Planning, Convexity of Tether Robot Workspace
\end{IEEEkeywords}

\section{Introduction}

\IEEEPARstart{T}{ethered} robots have natural advantages in maintaining stable power supply and communication links, which make them suitable to perform energy-consuming tasks and work in wireless communication-denied environments, for instance the coverage tasks~\cite{Shnaps2014Online}~\cite{Mechsy2017Novel}~\cite{Sharma20192}, disaster recovery missions~\cite{Pratt2008Use}, mountain climbing tasks~\cite{Abad2011Motion}~\cite{Tanner2013Online}, and exploration tasks~\cite{Shapovalov2020Exploration}. 
Yet the topological constraints introduced by the flexible tether have caused significant challenges in planning: 
The tethered robot may visit the same location with different tether states, regarded as different \textit{configurations}.
This indicates that the tethered path planning is essentially performed in the high dimensional \textit{workspace}, namely the set of all tether-length-admissible configurations, instead of the 2D collision-free environment. 

Existing algorithms indexed the configurations by the combination of the robot's location and the topological invariants of the tether curve such as $H$-signature~\cite{Bhattacharya2012Topological} and winding number~\cite{Pokorny2016High}. 
The workspace needs to be pre-calculated into a \textit{homotopy-augmented graph} before classic optimal graph searching algorithms such as A*~\cite{hart1968formal} were performed in the workspace~\cite{Kim2014Path}~\cite{Mccammon2017Planning}, wherein all nodes are tether-length-admissible robot configurations. 
The reader is referred to Section~\ref{section_related_works} for a detailed survey. 
Although the topological representation of the workspace is mathematically elegant, exhaustive pre-calculation of the workspace could be inefficient and consume a large amount of computational time, because the number of configurations in the workspace is much larger than the number of collision-free locations in the 2D environment, and a large proportion of the configurations in the workspace are not used by the path planner. 
~\cite{Kim2015Path} has also noticed the troublesome inefficiency of the workspace pre-calculation and therefore proposed the multi-heuristic A* algorithm, aiming to avoid pre-calculating the whole workspace, where however the resultant motion was only sub-optimal.

In this paper, we propose the first optimal tethered path planner without workspace pre-calculation where the key contributor is the workspace convexity, which is essentially the proof of no requirement to check the tether-length-admissibility of the optimal motion. 
A formal definition of the convexity will be presented in \textbf{Section~\ref{section_convexity}}. 
In other words, it is asserted and proven that if two configurations are in the workspace (i.e., admissible), then the commonplace (i.e., untethered) optimal motion between them must also be in the workspace (i.e., admissible). 
As such, any tethered optimal path must be the untethered locally optimal path whose homotopy is specified by the concatenation of the tether curve of the starting configuration and the goal configuration. 
This means that the workspace convexity has effectively transformed the optimal \textit{tethered reconfiguration} (TR) task between two configurations into a locally \textit{untethered path shortening} (UPS) task. 
Naturally generalising this idea to solve the classic \textit{optimal tethered planning} problem~\cite{Kim2014Path} whose goal configuration is undetermined, workspace pre-calculation is neither required, which is replaced by a \textit{goal configuration pre-calculation} (GCP) module which only collects the admissible configurations to visit the goal location. 
After admissible goal configurations are collected, all the optimal tethered reconfiguration motions can be easily constructed by executing UPS one time for each goal configuration, whereby the optimal tethered motion is the shortest reconfiguration motion.
Since convexity is a fundamental property of the tethered robot workspace, the above discussion could be further generalised to efficiently solve the following higher-level distance-optimal tethered path planning tasks which have not been solved before: 
\begin{enumerate}
\item Given a sequence of goals to be visited, the optimal tethered multi-goal visiting (TMV) task is solved. 
The WP module is replaced with the GCP module, and $\sum_{i=1}^{N-1}n_in_{i+1}$ UPS processes are then performed to construct all possible transitions to visit consecutive goals, where $N$ is the number of goals and $n_i$ is the number of configurations to visit the $i$-th goal. 
Finally, a numerical dynamic programming (DP) process is utilised to select the optimal TMV solution. 
\item Given a set of goal locations, the optimal tethered travelling salesman problem (TTSP), finding the shortest collision-free and tether-length-admissible cyclic path that visits all goals without a specific order, is solved. 
The solver consists of a GCP process, $(\sum_{i=1}^{N} n_i)^2 - \sum_{i=1}^Nn_i^2$ UPS processes, and a numerical generalised travelling salesman problem (GTSP)~\cite{Lien1993Transformation} process. 
\end{enumerate}

The following sections are organised as follows: 
Section~\ref{section_related_works} surveys existing literature. 
Section~\ref{section_problem_modelling} formally defines the terminologies and the problems to be solved in this paper. 
Section~\ref{section_simply_connected} presents the simply-connectedness of the workspace which clarifies the relation between local optimality and global optimality stated in this paper. 
Formal definition and proof of the workspace convexity are provided in Section~\ref{section_convexity}. 
Leveraging the convexity of the workspace, effective solutions to the optimal TR problem, optimal TP problem, optimal TMV problem, and optimal TTSP problem are presented in Section~\ref{section_solution}. 
Experiments are collected in Section~\ref{section_experiment}, with final concluding remarks gathered in Section~\ref{section_conclusion}. 

\section{Related Works}\label{section_related_works}

There have been extensive investigations into the distance-optimal tethered path planning problems. 
We survey them problem-wise below. 

\subsection{Tethered Path Planning in Obstacle-free Environments}
Early reports~\cite{Hert1995Moving}~\cite{Hert1996Ties}~\cite{Hert1999Motion} on tethered robot planning tasks mainly focused on obstacle-free environments, where each robot is tethered to a point on the boundary of the environment. 
Whilst a robot is moving to its goal location, the tether may be pushed and bent by other robots that contact it. The problem has been formulated in computational geometry, and algorithms with proven polynomial time complexity have been proposed. 

\subsection{Tethered Path Planning in Polygonal Environments}
Recently, some algorithms restricted the scope in polygonal environments~\cite{Teshnizi2014Computing}, where the algorithmic complexity can be calculated~\cite{Xavier1999Shortest}~\cite{Brass2015Shortest} as a polynomial of the number of straight segments in the starting configuration and the number of obstacle vertices. 
Besides, the tethered path planning problem has also been discussed in the visibility graph of a polygonal environment~\cite{Salzman2015Optimal}, where the algorithm is more efficient than grid-based approaches because the visibility graph~\cite{Ghosh1991Output} has effectively abstracted the topological information of the environment. 

\subsection{Tethered Path Planning in Arbitrary Planar Environments}
Another category of methods is to find optimal tethered robot motion in an arbitrary planar environment, represented in grid maps in particular. 
Research works have focused on exploring the workspace~\cite{Igarashi2010Homotopic}~\cite{Yershov2013Continuous}~\cite{Shnaps2014Online}, either in a resolutionally complete manner~\cite{Bhattacharya2010Search} or an asymptotically optimal manner~\cite{Pokorny2015Data}. 
~\cite{Kim2014Path} combined the homotopy-based topological approach and the graph search-based techniques to solve the optimal tethered path planning problem. 
The tethered robot workspace (represented in a homotopy augmented graph) was firstly calculated to collect all the tether-length-admissible robot configurations as well as their adjacency relations. 
Topological invariants such as \textit{$H$-signature}~\cite{Bhattacharya2012Topological} and winding number~\cite{Pokorny2016High} were adopted to differentiate the configurations with the same robot location. 
Classic graph searching-based planners such as A*~\cite{hart1968formal} and Dijkstra~\cite{Dijkstra1959Note} were then applied to the graph to generate the optimal path. 
Although the workspace pre-calculation is only executed once, it is still the most computationally inefficient process. 
To avoid pre-calculating the workspace,~\cite{Kim2015Path} proposed a multi-heuristic A* search technique which developed a dynamically generated set of heuristic functions to reduce the computational load. 
However, the solution of multi-heuristic A* had to be near optimal but not globally optimal. 

Besides discretising the environment as grids, sampling-based descriptions of the workspace may also be  legit~\cite{Wang2018Topological}. 
However, the solutions are still a two-stage process: A sampling-based roadmap, equivalent to the homotopy-augmented graph, was firstly constructed or obtained from data~\cite{Pokorny2015Data}, then a lifted graph-based path searching algorithm was applied in the roadmap to find the optimal path. 

As per the discussions above, concentrations have been mainly paid to the topological structure of the workspace, which is close to the algebraic topology~\cite{Rotman2013Introduction} in mathematics. 
It is worthwhile noting that an in-depth analysis of the geometric structure of the workspace, e.g., convexity, is of the same importance, which however is not a proposition that can be derived from a topological perspective. 
This is to be presented in this paper.

\subsection{Tethered Travelling Salesman Problem}
We also notice that~\cite{Mccammon2017Planning} proposed a non-entangled solution to the travelling salesman problem (TSP) of a tethered underwater vehicle. 
The algorithm consisted of a homotopy-aware path planning process and a mixed integer programming method for constructing a near-optimal TSP solution. 
The path planning algorithm adopted in~\cite{Mccammon2017Planning} has followed a similar vein as above~\cite{Kim2014Path}: A homotopy-aware probabilistic roadmap was pre-calculated~\cite{Kavraki1996Probabilistic}, and A*~\cite{hart1968formal} was utilised to find the optimal path. 
In particular, it has acknowledged~\cite{Mccammon2017Planning} that the algorithmic complexity to find all the topology-distinct robot motions between any two goals was too high, thus only a subset of homotopy classes were selected which led to non-optimality. 
In contrast, we prove in this paper that the reconfiguration motion can be simply obtained by a UPS process instead of a path searching process, hence is computationally affordable. 
This will eventually transform the optimal tethered TSP problem into a generalised travelling salesman problem (GTSP) in a reasonable computational time.  

\section{Problem Modelling}\label{section_problem_modelling}
In this section, we formally define the terminologies and the problems in this paper. 
We make the following assumptions: 
\begin{enumerate}
\item The environment is 2D. 
\item The robot is circular and omnidirectional. 
\item The tether is always taut. 
\item The tether-robot contacting point is the robot's centre. 
As such the endpoint of the tether is the robot's location.
\item The thickness of the tether is neglected. 
\end{enumerate}

\subsection{Terminologies and Notations}
We denote the 2D environment as $\mathcal{M}$, and the obstacle-free part of the environment as $\mathcal{M}_{\rm free}$. 
Given the radius of the circular robot, the collision-free part of the environment is denoted as $\mathcal{C}$. 
The base point is $p_0$ with $p_0\in \mathcal{C}$. 
The maximum length of the robot's tether is $L$. 

\begin{definition}
(Curve) A curve $\alpha$ is a one-parameter embedding of the unit interval into the environment, 
\begin{equation}
\alpha : [0, 1]\rightarrow \mathcal{M}_{\rm free},\ s\mapsto \alpha(s)
\end{equation}
The backtracking of curve $\alpha$ is defined as 
\begin{equation}
\alpha^{-1}:[0, 1]\rightarrow \mathcal{M}_{\rm free},\ s\mapsto \alpha(1-s)
\end{equation}
The concatenation of two curves is denoted as $\alpha_1*\alpha_2$, 
\begin{equation}
(\alpha_1*\alpha_2)(s) = \left\{
\begin{aligned}
&\alpha_1(2s),\ &0\leq s\leq \frac{1}{2}\\
&\alpha_2(2s-1),\ &\frac{1}{2}\leq s\leq 1
\end{aligned}
\right.
\end{equation}
For clarity, a curve is always represented by a Greek alphabet in this paper.
If $\alpha$ lies totally in $\mathcal{C}$, we write $\alpha\subset \mathcal{C}$. 
The length of $\alpha$ is denoted by $g(\alpha)$. 
\end{definition}

\begin{definition}
(Homotopy) Given two curves $\alpha_1$ and $\alpha_2$ with $\alpha_1(0) = \alpha_2(0)\in \mathcal{M}_{\rm free}$ and $\alpha_1(1) = \alpha_2(1)\in \mathcal{M}_{\rm free}$, they are homotopic, denoted by $\alpha_1\simeq \alpha_2$, if one can be continuously deformed into the other one in $\mathcal{M}_{\rm free}$ with endpoints fixed. 
\end{definition}

\begin{definition}
(Local Shortening of Curves) Given a curve $\alpha$, the locally shortening of $\alpha$ in $\mathcal{M}_{\rm free}$ (or $\mathcal{C}$) is denoted as $S_\alpha^{\mathcal{M}_{\rm free}}$ (or $S_\alpha^{\mathcal{C}}$).
\end{definition}

\begin{definition}
(Configuration) The configuration $c$ of a tethered robot consists of the robot's location $p$ and the state of its tether represented by a curve $\alpha$
\begin{equation}
\begin{aligned}
c&\triangleq (p, \alpha)\\
s.t.~~~~ \alpha(0) &= p_0,\ \alpha(1) = p,\ \alpha\subset \mathcal{M}_{\rm free},\ p\in \mathcal{C}
\end{aligned}
\end{equation}
The orientation of $\alpha$ is from $p_0$ to $p$. 
In particularly, we denote the home configuration as $(p_0, \alpha_0)$, where $\alpha_0$ is the constant mapping $[0, 1]\mapsto p_0$. 
\end{definition}

\begin{remark}
(Admissible)
Note that the definition of the configuration itself does not guarantee the feasibility of tether length, and in this paper we may use some configurations whose tether violates the maximum tether length constraint. 
For easy reference, if the length of the tether of a configuration is less than the maximum allowable tether length $L$, we say the configuration is (tether-length-)admissible. 
Otherwise, we sometimes say it is a \textit{constraint-violated configurations} to emphasise that it violates the maximum tether length constraint. 
\end{remark}

\begin{remark}\label{rem:path_along_a_tether}
(Homotopic Collision-free Path along a Valid Tether) 
Generally, a curve $\alpha$ may not be in $\mathcal{C}$, then $S_\alpha^{\mathcal{C}}$ is undefined.  
However, if $\alpha$ represents a tether state, then the robot must be able to travel along a collision-free path homotopic to the tether. 
So the locally shortest collision-free path homotopic to a given tether curve $\alpha$ always exists and is well-defined, denoted as $S^{\mathcal{C}}_{\alpha}$. 
\end{remark}

\begin{definition}
(Configurations Induced by Motion) Given a configuration $c_1 = (p_1, \alpha_1)$ and a collision-free path $\beta$, $\beta(0) = p_1$. 
Whilst the robot is moving along $\beta\subset \mathcal{C}$, the cable simultaneously deforms and keeps taut. 
Then the interim configurations induced by $\beta$ (from $c_1$) are well-defined as 
\begin{equation}
(\beta(s), S^{\mathcal{M}_{\rm free}}_{\alpha_1*\beta([0, s])}), s\in [0, 1]
\end{equation}
If all induced configurations are admissible,  we say the motion is (tether-length-)admissible. 
\end{definition}

\begin{definition}
(Workspace) The workspace is the set of all admissible configurations. 
It is a subset of the set of all configurations without considering the maximum tether length constraint. 
\end{definition}

\begin{remark}\label{rem:finite_number_of_configurations}
Assuming that the tether is taut, for each reachable location, there are only a finite number of admissible configurations to visit it. 
\end{remark}

\subsection{Problem Statement}
The solution to the following path planning problems for a tethered omni-directional robot will be presented in this paper. 

\begin{problem}\label{prob:TR}
(Optimal Tethered Reconfiguration, TR) Given the starting configuration $c_1 = (p_1, \alpha_1)$, $g(\alpha_1)\leq L$, and the goal configuration $c_2 = (p_2, \alpha_2)$, $g(\alpha_2)\leq L$, the tethered reconfiguration motion is a collision-free path $\beta$ from $p_1$ to $p_2$, such that all induced configurations satisfy the maximum tether length constraint, and that $c_2$ is the induced target configuration, 
\begin{equation}
\begin{aligned}
g(S_{\alpha_1*\beta([0, s])}^{\mathcal{M}_{\rm free}}) &\leq L, \forall s\in [0, 1]\\
\alpha_2 &\simeq \alpha_1*\beta,\ \beta\subset \mathcal{C}
\end{aligned}
\end{equation}
The optimal tethered reconfiguration solution is the shortest path among all reconfiguration motions.
\end{problem}

\begin{problem}\label{prob:TP}
(Classic Optimal Tethered Planning, TP) Given the starting configuration $c_1 = (p_1, \alpha_1)$ and the goal location $p_2$, we denote the set of admissible configurations visiting $p_2$ as (say there are $n$ ones, where $n$ is finite by \textbf{Remark~\ref{rem:finite_number_of_configurations}}): 
\begin{equation}
\{c_{21} \triangleq (p_2, \alpha_{21}), \cdots, c_{2n} \triangleq (p_2, \alpha_{2n})\}
\end{equation}
Then, there may exist up to (here we have not asserted that the reconfiguration motion to any of the goal configurations exists) $n$ reconfiguration motions from $c_1$ to each possible target configuration $\{c_{21}, \cdots, c_{2n}\}$, denoted by $\{\beta_1, \cdots, \beta_n\}$. 
Then the optimal tethered path solution is the shortest one among them, 
\begin{equation}
\beta^* = \mathop{\rm argmin}\limits_{\beta_i\in \{\beta_i\}_{i=1}^n} g(\beta)
\end{equation}
\end{problem}
Note that the statement of the optimal TP problem is mathematically equivalent to previous formulations such as~\cite{Kim2014Path}. 
Here it is re-formulated as above for the easy usage of an efficient optimal TR solution. 

The following problems have not been optimally solved because of their high complexity. 
These will also be solved in this paper. 

\begin{problem}\label{prob:TMV}
(Optimal Tethered Multi-Goal Visiting, TMV) Given the starting configuration as the home configuration and a sequence of goal locations $p_1, \cdots, p_N$, we denote the set of admissible configurations to visit $p_i$ as (say there are $n_i$ ones, where $n_i$ are finite by \textbf{Remark~\ref{rem:finite_number_of_configurations}}):
\begin{equation}
C_i\triangleq\{c_{ij_i} \triangleq (p_i, \alpha_{ij_i})\},\ j_i = 1, \cdots, n_i,\ i = 1, \cdots, N
\end{equation}
To emulate a practical cyclic target monitoring task, we formally let the $0$-th and $(N+1)$-th goal be the base point itself, with the home configuration being the only available configuration, i.e., 
\begin{equation}\label{eqn:0andn+1}
\begin{aligned}
&n_0 = n_{N+1} = 1\\
&C_0 = \{c_{0j_0}\} \triangleq \{(p_0, \alpha_{j_0})\} = \{(p_0, \alpha_{0})\}\\
&C_{N+1} = \{c_{(N+1)j_{N+1}}\} \triangleq \{(p_{N+1}, \alpha_{j_{N+1}})\} = \{(p_0, \alpha_{0})\}
\end{aligned}
\end{equation}
A valid tethered multi-goal visiting path $\beta$ is the concatenation of reconfiguration motions
\begin{equation}
\beta = \beta_{01}*\beta_{12}*\cdots*\beta_{(N-1)N}*\beta_{N(N+1)}
\end{equation}
where 
\begin{equation}
\beta_{ij}(0) =p_i,\ \beta_{ij}(1) =p_j
\end{equation}
The optimal TMV solution is the shortest TMV path.
\end{problem}
Obviously, the path segments $\beta_{ij}$ in \textbf{Problem~\ref{prob:TMV}} are all optimal TR solutions, or else a better TMV solution can be simply constructed by further optimising the non-optimal reconfiguration motion segments.

Finally, the optimal \textit{tethered travelling salesman problem} (TTSP) is formulated as 

\begin{problem}\label{prob:TTSP}
(Optimal Tethered Travelling Salesman Problem, TTSP) 
Let $S^N$ represent the set of all permutations of $N$ objects, i.e., 
\begin{equation}
\sigma: \{1, \cdots, N\}\rightarrow \{1, \cdots, N\},\ \sigma(i) \neq \sigma(j) \mbox{ when } i\neq j
\end{equation}
Given the same definitions as in \textbf{Problem~\ref{prob:TMV}}, each permutation indicates a visiting order of the $N$ goals, 
\begin{equation}
\sigma\mapsto [p_{\sigma(1)}, \cdots, p_{\sigma(N)}]
\end{equation}
which admits an optimal TMV solution $\beta^*_\sigma$. 
Then the optimal tethered travelling salesman problem is to find the shortest travelling path
\begin{equation}
\beta^{**} = \mathop{\rm argmin}\limits_{\beta^*_{\sigma}\in \{\beta^*_{\sigma}, \sigma\in S^N\}}g(\beta^*_{\sigma})
\end{equation}
\end{problem}

Notations used in this paper are summarised in Table.~\ref{table:symbols}.

\begin{table}[t]
\small\sf\centering
\caption{List of Notations}\label{table:symbols}
\begin{tabular}{>{\bfseries}ll}
\hline
\normalfont{Symbols}   &    Meaning   \\
\hline
$\mathcal{M}$ & The 2D environment\\
$\mathcal{M}_{\rm free}$ & Obstacle-free environment\\
$\mathcal{C}$ & Collision-free environment\\
$c$ & A configuration which contains \\
 &  \qquad robot's location and the tether shape\\
$g(\cdot)$ & Length of a curve\\
$L$ & Maximum tether length\\
$N$ & Number of goals\\
$n_i, i = 1, \cdots, N$ & The number of tether-length-admissible \\
& \qquad configurations to visit the $i$-th goal\\
$p_0$ & Base point\\
$p_1, \cdots, p_N$ & Goal locations\\
$\alpha, \beta, \gamma$ & A curve\\
$S_{\alpha}^*$ & The local shortening of curve $\alpha$\\
&\qquad in the $*$ environment\\
\hline
\end{tabular}
\end{table}

\section{Simply-Connectedness of the Tethered Robot Workspace}\label{section_simply_connected}
Different from the Euclidean scenario where any convex set must also be simply-connected, the simply-connectedness of the tethered robot workspace is not a direct corollary of its convexity, because the connection between two configurations in the workspace is not the straight line segment. 
In this section, as an easy remark, we prove that the workspace is indeed simply-connected. 
\begin{remark}\label{rem:simply-connected}
(Simply-Connectedness) For any two configurations $(p_1, \alpha_1)$ and $(p_2, \alpha_2)$, all the reconfiguration motions are homotopic.
\end{remark}
\begin{proof}
Proof by contradiction. Assume there exist two reconfiguration motions which are non-homotopic, denoted as $\beta_1$ and $\beta_2$. 
Since the path concatenation preserves homotopy, we have
\begin{equation}
\beta_1\not\simeq \beta_2\Rightarrow \alpha_1^{-1}*\beta_1\not\simeq \alpha_1^{-1}*\beta_2
\end{equation}
The local path shortening process also preserves homotopy, thus
\begin{equation}\label{eqn:homotopic}
\alpha_1^{-1}*\beta_1\not\simeq \alpha_1^{-1}*\beta_2\Rightarrow S_{\alpha_1^{-1}*\beta_1}^{\mathcal{M}_{\rm free}}\not\simeq S_{\alpha_1^{-1}*\beta_2}^{\mathcal{M}_{\rm free}}
\end{equation}
However, by assumption both $\beta_1$ and $\beta_2$ are reconfiguration motions, i.e., 
\begin{equation}\label{eqn:nonhomotopic}
S_{\alpha_1^{-1}*\beta_1}^{\mathcal{M}_{\rm free}} = \alpha_2 = S_{\alpha_1^{-1}*\beta_2}^{\mathcal{M}_{\rm free}}\Rightarrow S_{\alpha_1^{-1}*\beta_1}^{\mathcal{M}_{\rm free}}\simeq S_{\alpha_1^{-1}*\beta_2}^{\mathcal{M}_{\rm free}}
\end{equation}
Eqn.~(\ref{eqn:homotopic}) and Eqn.~(\ref{eqn:nonhomotopic}) contradicts. 
\end{proof}
Based on \textbf{Remark~\ref{rem:simply-connected}}, as long as we find one valid reconfiguration motion for the given two configurations, the optimal reconfiguration motion must be homotopic to it. 
In this paper, for rigorousness and intuitive reference, we still call the formal variable ``$S^\mathcal{C}_*$" as a \textit{locally} shortest curve, but the reader should be noted that, as the motion between two configurations, it is also the globally shortest one. 

\section{Convexity of The Tethered Robot Workspace}\label{section_convexity}

Before we prove the convexity, we identify what it is. 
Recall that for a convex set of points embedded in the Euclidean space, the convexity means that given any two points in the set, all the points in the shortest (straight) curve connecting the two points in the Euclidean space also lie in the set. 
For tethered robot planning task, the workspace is the set of all admissible configurations, whose ambient space is the set of all configurations with no tether length constraint. 
Given two configurations are admissible, regarded as points in the workspace, the workspace convexity asserts that the optimal sequence of configurations that connect the given two configurations in the ambient space (where the maximum tether length constraint is ignored) must also be in the workspace.

Concretely, let $(p_1, \alpha_1)$ and $(p_2, \alpha_2)$ be two configurations in the workspace. 
We pay attention to a special curve, 
\begin{equation}\label{eqn:beta}
\beta \triangleq S_{\alpha_1^{-1}*\alpha_2}^\mathcal{C}\triangleq S_{(S_{\alpha_1}^\mathcal{C})^{-1}*S_{\alpha_2}^\mathcal{C}}^\mathcal{C}
\end{equation}
It is well-defined, because by \textbf{Remark~\ref{rem:path_along_a_tether}} both $S_{\alpha_1}^\mathcal{C}$ and $S_{\alpha_2}^{\mathcal{C}}$ exist. 
Then $\beta$ is the collision-free path generated by locally shortening their concatenation, $(S_{\alpha_1}^\mathcal{C})^{-1}*S_{\alpha_2}^\mathcal{C}$.  
Since the curve shortening process is not subject to the maximum tether length constraint, literally there might be some induced configurations that are constraint-violated configurations. 
However, we prove that this is impossible, i.e., all the induced configurations are proven admissible. 
From another perspective, the workspace convexity states the relation between tethered optimal paths and untethered locally optimal paths: A tethered shortest path is always the untethered locally shortest path in a certain homotopy class of paths. \footnote{
~\cite{Kim2014Path} has presented similar wordings that the shortest tethered resultant path may be ``the shortest trajectory in some other homotopy class". 
However, the authors did not explicitly state whether they were referring to a trajectory whose local optimality is \textbf{not} subject to the maximum tether length constraint. 
Even if they claim the same as our \textbf{Theorem~\ref{thm:convexity}}, the proposition is non-trivial but there was no formal proof. 
Moreover, they did not use an untethered path shortening module in their algorithm. 
So we regard the similar wordings in~\cite{Kim2014Path} as informal references to ``the shortest one in a set of homotopic admissible motions". } 
Some alternative paths need to be constructed for the proof of convexity, presented next. 

\subsection{Construction of Alternative Path}
Let $(p_1, \alpha_1)$ and $(p_2, \alpha_2)$ be two admissible configurations, for any collision-free path $\gamma$ connect $p_1$ to $p_2$ such that $\alpha_1*\gamma \simeq \alpha_2$, we construct an alternative collision-free path $\tilde{\gamma}$.

\begin{definition}\label{def:tildebeta}
($\tilde{\gamma}$) A collision-free curve $\tilde{\gamma}$ is defined as follows: 
\begin{equation}\label{eqn:tildebeta}
\tilde{\gamma}:[0, 1]\rightarrow \mathcal{C},\ s\mapsto \tilde{\gamma}(s) =  S_{\alpha_1*\gamma([0, s])}^{\mathcal{C}}(\tilde{s})
\end{equation}
where $\tilde{s} = \tilde{s}(s)\in [0, 1]$ depends on $s$ such that: 
\begin{enumerate}
\item If $\gamma$ is admissible, i.e., $g(S^{\mathcal{M}_{\rm free}}_{\alpha_1*\gamma([0, s])})\leq L$, then $\tilde{s} = 1$ (i.e., $\tilde{\gamma}(s) = \gamma(s)$). 
\item If the configuration induced by $\gamma$ is constraint-violated, then $\tilde{s}$ is chosen such that the robot stays in the locally shortest collision-free path, and that the length of the tether of $\tilde{\gamma}$ is exactly $L$, 
\begin{equation}
g(S^{\mathcal{M}_{\rm free}}_{S_{\alpha_1*\beta([0, s])}^\mathcal{C}([0, \tilde{s}])}) = L
\end{equation}
\end{enumerate}
\end{definition}

\begin{remark}
$\tilde{\gamma}$ is a continuous curve, because 
\begin{enumerate}
\item For the path segments where $\gamma$ and $\tilde{\gamma}$ coincide, $\tilde{\gamma}$ is continuous because $\gamma$ is continuous. 
\item For the path segments where $\gamma$ and $\tilde{\gamma}$ are different, $\tilde{\gamma}$ is the boundary of the reachable region of a tethered robot whose tether length is $L$, hence is also continuous. 
\item Finally, at the point where $\gamma$ and $\tilde{\gamma}$ bifurcate, denote the induced configuration as $c$. 
On one side, $\gamma$ and $\tilde{\gamma}$ coincide, thus $c$ is connected to $\tilde{\gamma}$. 
On the other side, the tether length of $c$ is exactly $L$, thus $c$ is also connected to $\tilde{\gamma}$. 
\end{enumerate}
Hence the curve $\tilde{\gamma}$ is indeed continuous. 
\end{remark}

See Fig.~\ref{fig:preserving} for an illustration of $\tilde{\gamma}$ for a given $\gamma$. 

\begin{lemma}\label{lem:tildebeta_is_reconfig}
$\tilde{\gamma}$ is homotopic to $\gamma$. 
\end{lemma}
\begin{proof}
This is correct because the following continuous mapping exists: 
\begin{equation}
\begin{aligned}
&F: [0, 1]\times [0, 1]\rightarrow \mathcal{C}\\
&(s, t)\mapsto S^{\mathcal{C}}_{\alpha_1*\gamma([0, s])} (\tilde{s} + (1-\tilde{s})t)
\end{aligned}
\end{equation}
where $\tilde{s}$ depends on $s$, as defined in~Eqn.~(\ref{eqn:tildebeta}). 
It is easy to verify that $F(\cdot, 1) = \gamma$ and $F(\cdot, 0) = \tilde{\gamma}$. 
And the superscript $\mathcal{C}$ indicates that all the points are collision-free, hence also obstacle-free. 
\end{proof}

We need the following easy lemma to prove that $\tilde{\gamma}$ is shorter than $\gamma$. 
\begin{lemma}\label{lem:liu}
The final segment of the tether from the last tether-obstacle contact point to the robot's centre is straight. (By assumption the tether is taut. )
\end{lemma}
\begin{proof}
The robot's centre must be within the robot's footprint, i.e., obstacle-free. 
So the final segment of the tether is not a wall-following curve. 
Then it is straight. 
\end{proof}

\begin{figure}[t]
\centering
\includegraphics[width=0.30\textwidth]{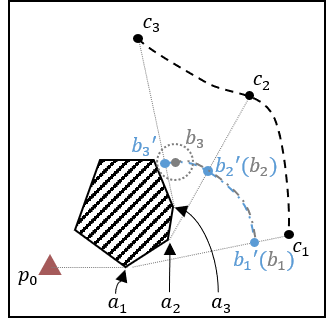}
\caption{Illustration of the paths $\gamma$, $\tilde{\gamma}$, and $\tilde{\gamma}'$. 
The last obstacle-tether contact points are depicted by $a_1, a_2$, and $a_3$. 
$\tilde{\gamma}$ is represented by $b_1$, $b_2$, and $b_3$. 
$\tilde{\gamma}'$ is represented by $b_1'$, $b_2'$, and $b_3'$. 
$\gamma$ is represented by $c_1$, $c_2$, and $c_3$. 
}\label{fig:symbol}
\end{figure}

\begin{figure}[t]
\centering
\includegraphics[width=0.48\textwidth]{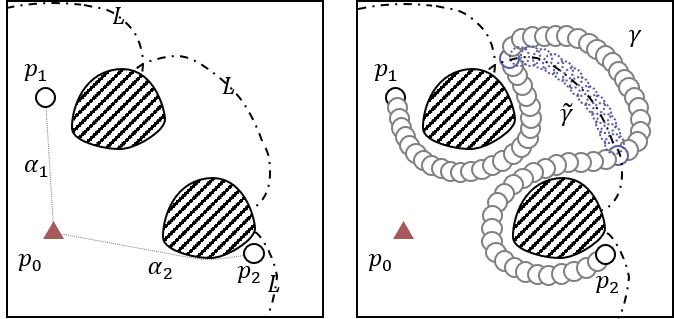}
\caption{
Illustration of a robot motion $\gamma$ and its alternative path $\tilde{\gamma}$. 
Left: 
Illustration of the environment, the base point location, the starting configuration, the goal configuration, and the reachable region given the maximum tether length $L$. 
Right: 
Let a tethered robot motion $\gamma$ violate the maximum tether length constraint, shown as a sequence of grey circles. 
Then we can construct $\tilde{\gamma}$ which is an alternative path of $\gamma$. $\tilde{\gamma}$ will be proven homotopic to $\gamma$, complied with the maximum tether length constraint, and shorter than $\gamma$. 
}\label{fig:preserving}
\end{figure}

Besides, we need to construct another curve for the proof. 
\begin{definition}
($\tilde{\gamma}'$) An obstacle-free curve $\tilde{\gamma}'$ is defined as follows: 
\begin{equation}
\tilde{\gamma}': [0, 1]\rightarrow \mathcal{M}_{\rm free},\ s\mapsto b'\triangleq \tilde{\gamma}'(s) = S^{\mathcal{M}_{\rm free}}_{\alpha_1*\gamma([0, s])}(s_b')
\end{equation}
where $s_b' = s_b'(s)\in [0, 1]$ depends on $s$ such that:
\begin{enumerate}
\item If $\gamma$ is admissible, i.e., $g(S^{\mathcal{M}_{\rm free}}_{\alpha_1*\gamma([0, s])})\leq L$, then $s_b' = 1$ (i.e., $\tilde{\gamma}'(s) = \gamma(s)$). 
\item If the configuration induced by $\gamma$ is constraint-violated, then $s_b'$ is chosen such that the length of the truncated part of the tether from the base point to $\tilde{\gamma}'(s)$ is exactly $L$, i.e., 
\begin{equation}
g(S^{\mathcal{M}_{\rm free}}_{\alpha_1*\gamma([0, s])}([0, s_b'])) = L
\end{equation}
\end{enumerate}
An illustration of the curve $\tilde{\gamma}'$ is provided in Fig.~\ref{fig:symbol}. 
\end{definition}

\begin{lemma}\label{lem:tildegamma_is_shorter_than_tildegamma_prime}
$\tilde{\gamma}$ is shorter than $\tilde{\gamma}'$. 
\end{lemma}

\begin{proof}
We compare the path length of $\tilde{\gamma}$ and $\tilde{\gamma}'$ by comparing the length of their infinitesimal path segment, $g(\Delta\tilde{\gamma}) \triangleq g(\tilde{\gamma}([s, s+\Delta s]))$ and $g(\Delta \tilde{\gamma}') \triangleq g(\tilde{\gamma}'([s, s+\Delta s]))$. 
\begin{enumerate}
\item For the $s$ where the configuration induced by $\gamma$ is admissible,  $\tilde{\gamma}$ and $\tilde{\gamma}'$ coincide, thus $g(\Delta\tilde{\gamma}) = g(\Delta\tilde{\gamma}')$. 
\item Otherwise, 
\begin{enumerate}
\item When $\tilde{\gamma}'(s)$ is collision-free, we still have $\tilde{\gamma}'(s) = \tilde{\gamma}(s)$. Thus $g(\Delta\tilde{\gamma}) = g(\Delta\tilde{\gamma}')$. 
\item Or else, with $s$ increasing, $\tilde{\gamma}(s)$ remain static at the last collision-free point, whilst $\tilde{\gamma}'(s)$ is still varying in $\mathcal{M}_{\rm free}$, thus $0 = g(\Delta\tilde{\gamma}') < g(\Delta \tilde{\gamma})$.  
\end{enumerate}
\end{enumerate}

\end{proof}

\begin{theorem}\label{thm:shorter_length}
$\tilde{\gamma}$ is shorter than $\gamma$. 
\end{theorem}

\begin{proof}
We first prove the claim in an arbitrary polygonal environment. Then it is also correct in any arbitrary environment because a sufficiently fine polygonal approximation of the environment always exists. 

Assume that all obstacles in $\mathcal{M}$ are polygonal. 
We prove the claim by showing that $g(\tilde{\gamma}) \leq g(\tilde{\gamma}')\leq g(\gamma)$, where the equalities hold only when $\gamma$ is admissible. 
The former inequality has been proven in \textbf{Lemma~\ref{lem:tildegamma_is_shorter_than_tildegamma_prime}}, and the latter one is verified by comparing the length of their infinitesimal path segment, $g(\Delta \tilde{\gamma}') \triangleq g(\tilde{\gamma}'([s, s+\Delta s]))$ and $g(\Delta \gamma) \triangleq g(\gamma([s, s+\Delta s]))$. 
For easy notation, we denote $a = a(s)$ as the last obstacle vertex that the tether hits, $b' = b'(s)$ as $\tilde{\gamma}'(s)$, and $c = c(s)$ as $\gamma(s)$. 
The length of the truncated tether from the base point to $a$ is denoted as $L_a$. 
\begin{enumerate}
\item For the $s$ where the configuration induced by $\gamma$ is admissible, $\gamma$ and $\tilde{\gamma}'$ coincide, thus $g(\Delta\tilde{\gamma}') = g(\Delta \gamma)$. 
\end{enumerate}

\begin{enumerate}
\item When $L \leq L_a$, the truncated tether from the base point to $a$ remains static, thus $g(\Delta\tilde{\gamma}') = 0$. 
In contrast, $g(\Delta \gamma) >  0$. 
Hence $g(\Delta\tilde{\gamma}') < g(\Delta \gamma)$. 
\item When $L > L_a$, by \textbf{Lemma~\ref{lem:liu}} the points $a$, $b'$, and $c$ are always co-linear. 
Thus the motion of $b'$ and $c$ have the same angular velocity with respect to $a$. 
As such, $c$ has a longer distance to $a$, so the motion of $c$ is longer than that of $b'$, thus $g(\Delta\tilde{\gamma}') < g(\Delta\gamma)$. 

\end{enumerate}

\end{proof}

\subsection{Proof of Convexity}

\begin{theorem}\label{thm:convexity}
(Convexity of Workspace) Given any two configurations $(p_1, \alpha_1)$ and $(p_2, \alpha_2)$ in the workspace, define $\beta$ as a locally shortest collision-free path: 
\begin{equation}
\beta\triangleq S^{\mathcal{C}}_{\alpha_1^{-1}*\alpha_2}
\end{equation}
then $\beta$ is admissible, i.e., 
\begin{equation}
g(S^{\mathcal{M}_{\rm free}}_{\alpha_1*\beta([0, s])}) \leq L, \forall s\in [0, 1]
\end{equation}
\end{theorem}

\begin{proof}
Proof by contradiction. 
If $\beta$ is not admissible, we can construct an alternative path $\tilde{\beta}$ as per \textbf{Definition~\ref{def:tildebeta}}. 
$\tilde{\beta}$ has been proven homotopic to $\beta$ by \textbf{Lemma~\ref{lem:tildebeta_is_reconfig}}, and shorter than $\beta$ by \textbf{Theorem~\ref{thm:shorter_length}}. 
This violates the assumption that $\beta$ is an untethered locally shortest path. 
So $\beta$ must be admissible. 
See Fig.~\ref{fig:preserving} for illustration. 
\end{proof}

\section{Solutions to Problems}\label{section_solution}
In this section, we provide efficient solutions to the optimal TR, optimal TP, optimal TMV, and optimal TTSP.

\subsection{Optimal Tethered Reconfiguration}

In this subsection, the solution to \textbf{Problem~\ref{prob:TR}} is proven attainable by an untethered path shortening (UPS) process, e.g., the naive local shortcut mechanism in 2D, instead of path searching in the workspace. 
\begin{theorem}\label{thm:TR}
Given the starting configuration $(p_1, \alpha_1)$, $g(\alpha_1)\leq L$, and the target configuration $(p_2, \alpha_2)$, $g(\alpha_2)\leq L$, the optimal tethered reconfiguration motion is 
\begin{equation}
\beta^* = S^{\mathcal{C}}_{\alpha_1^{-1}*\alpha_2}
\end{equation} 
\end{theorem}
\begin{proof}

On one hand, since both $(p_1, \alpha_1)$ and $(p_2,\alpha_2)$ are admissible, by \textbf{Theorem~\ref{thm:convexity}}, $S^{\mathcal{C}}_{\alpha_1^{-1}*\alpha_2}$ is admissible. 

On the other hand, by the definition of reconfiguration motion, we have 
\begin{equation}
\begin{aligned}
&\alpha_1*\beta^*\simeq \alpha_2\\
\Rightarrow& \alpha_1^{-1}*\alpha_1*\beta^*\simeq \alpha_1^{-1}*\alpha_2\\
\Rightarrow& \beta^*\simeq \alpha_1^{-1}*\alpha_2\\
\Rightarrow& \beta^*\simeq S^{\mathcal{C}}_{\alpha_1^{-1}*\alpha_2}
\end{aligned}
\end{equation}

Finally, note that $S^{\mathcal{C}}_{\alpha_1^{-1}*\alpha_2}$ is the locally shortest path whilst $\beta^*$ needs to additionally satisfy the maximum tether length constraint, we have
\begin{equation}
g(S^{\mathcal{C}}_{\alpha_1^{-1}*\alpha_2}) \leq g(\beta^*)
\end{equation}
Hence $\beta^* = S^{\mathcal{C}}_{\alpha_1^{-1}*\alpha_2}$. 

\end{proof}

In practical, the configurations $(p_1, \alpha_1)$ and $(p_2, \alpha_2)$ are usually given after they have been observed to be reachable, whereby the collision-free paths $S^{\mathcal{C}}_{\alpha_1}$ and $S^{\mathcal{C}}_{\alpha_2}$ have been known. 
Hence only the untethered path shortening process needs to be performed.

Besides, as a simple application of the workspace convexity, the following corollaries may also inspire the tethered robot planning community. 
\begin{corollary}\label{cor:max_length_point}
Given two configurations $(p_1, \alpha_1)$ and $(p_2, \alpha_2)$, the minimum tether length for a robot to perform the optimal reconfiguration is $\max\{g(\alpha_1), g(\alpha_2)\}$. 
\end{corollary}
\begin{proof}
Let $L^*$ denote the minimum required tether length of the robot to fulfil the given optimal reconfiguration motion. 
Setting the tether length as $\max\{g(\alpha_1), g(\alpha_2)\}$, the configurations $(p_1, \alpha_1)$ and $(p_2, \alpha_2)$ are still admissible. 
Then by \textbf{Theorem~\ref{thm:convexity}}, the collision-free motion $S^{\mathcal{C}}_{\alpha_1^{-1}*\alpha_2}$ is admissible. 
Hence $L^*\leq \max\{g(\alpha_1), g(\alpha_2)\}$.  
On the other hand, if $L^*$ is smaller than $\max\{g(\alpha_1), g(\alpha_2)\}$, then either the starting configuration or the goal configuration would be unreachable. 
Hence $L^*\geq \max\{g(\alpha_1), g(\alpha_2)\}$. 
In summary, $L^* = \max\{g(\alpha_1), g(\alpha_2)\}$.

\end{proof}

Motivated by the efficient optimal TR solution, when solving high-level optimal tethered path planning problems, to avoid pre-calculating the workspace, a generic multi-stage pipeline is presented as follows: 
\begin{enumerate}
\item (Goal Configuration Pre-Calculation, GCP)
We identify the admissible configuration to visit each goal location. 
If the configuration is non-unique, we find all of them (by \textbf{Remark~\ref{rem:finite_number_of_configurations}}, there are only finite ones).  
Since the pre-calculation is restricted to goal locations only, GCP could be much more efficient than WP.

\item (Optimal Reconfiguration Construction) 
For any two configurations, the optimal tethered reconfiguration motion is constructed by an untethered path shortening (UPS) process. 
This process is also efficient because it is performed in the 2D space instead of in the workspace, and more importantly, the maximum tether length constraint is neglected. 
\item (Constructing the Optimal Solution) After all reconfiguration motions are constructed, they are concatenated to form the tethered robot planning task solutions. 
And the optimal solution is the shortest one among all the valid solutions. 
Non-enumerative techniques such as the Dynamic Programming (DP) approach could further speed up this step.
\end{enumerate}

\subsection{Solution to Optimal Tethered Planning}
The multi-stage process to solve \textbf{Problem~\ref{prob:TP}} is stated as follows: 
 
\textbf{Stage 1: Goal Configuration Pre-Collection.} 
We perform a GCP process to $p_2$. 
The goal configurations are formally written as (say there are $n$ ones) 
\begin{equation}
\begin{aligned}
&\{c_{21}, \cdots, c_{2n}\}\\
c_{2i}\triangleq &(p_2, \alpha_{2i}),\ i = 1, \cdots, n
\end{aligned}
\end{equation}

\textbf{Stage 2: Optimal Reconfiguration Construction. } 
From $c_1$ to each $c_{2i}$, the optimal tethered reconfiguration motion is obtained using the optimal TR solution discussed above, denoted as 
\begin{equation}
\begin{aligned}
&\{\beta_{1}, \cdots, \beta_{n}\}\\
\beta_i \triangleq &\alpha^{\mathcal{C}}_{\alpha_1^{-1}*\alpha_{2i}},\ i = 1, \cdots, n
\end{aligned}
\end{equation}

\textbf{Stage 3: Construct the Optimal Solution. }
The least-cost reconfiguration motion is the desired solution
\begin{equation}
\beta^* = \mathop{\rm argmin}\limits_{\beta\in \{\beta_i\}} g(\beta)
\end{equation}

\begin{algorithm}[t]
    \caption{Optimal Tethered Multi-Goal Visiting Planner}\label{alg:TMV}
    \begin{algorithmic}[1]
        \Require Environment $\mathcal{M}$, Initial configuration $c_0 = (p_0, \alpha_0)$, goal points $p_1, \cdots, p_N$ 
        \Ensure The optimal path $\beta^*$
\State // Goal Configuration Pre-Calculation
\State $n_0= 1$, $c_{01} \triangleq (p_0, \alpha_{01}) = c_0$
\State $\{c_{ij_i} \triangleq (p_i, \alpha_{ij_i})| j_i = 1, \cdots, n_i, i = 1, \cdots, N\}$
\Statex \qquad $\gets$ ShortestNonhomotopicPathPlanner($p_1, \cdots, p_N$)
\State $n_{N+1}= 1$, $c_{(N+1)1}\triangleq (p_{N+1}, \alpha_{(N+1)1}) = c_0$
\State // Solving All Optimal TRs
\For{$i = 0, \cdots, N$}
	\State Initialise $F_{i(i+1)}$ and $\Gamma_{i(i+1)}$
	\State // The dimension of $F_{i(i+1)}$ and $\Gamma_{i(i+1)}$ is $n_i\times n_{i+1}$
	\For{$j = 1, \cdots, n_i$}
		\For{$k = 1, \cdots, n_{i+1}$}
			\State $\Gamma_{i(i+1)}(j, k) \gets S^{\mathcal{C}}_{\alpha_{ij}^{-1}*\alpha_{(i+1)k}}$
			\State $F_{i(i+1)}(j, k) \gets g(\Gamma_{i(i+1)}(j, k))$
		\EndFor
	\EndFor
\EndFor
\State // Construct the Optimal Solution
\For{$i = 0, \cdots, N$}
	\For{$j = 1, \cdots, n_{i+1}$}
		\State $\Gamma_{(i+1)j}^* = \Gamma_{i1}^**\Gamma_{i(i+1)}(1, j)$
		\State $F_{(i+1)j}^* = g(\Gamma_{(i+1)j}^*)$ 
		\For{$k = 2, \cdots, n_i$}
			\State path\_temp $= \Gamma_{ik}^**\Gamma_{i(i+1)}(k, j)$
			\State cost\_temp $= F_{ik}^* + F_{i(i+1)}(k, j)$
			\If{cost\_temp $< F_{(i+1)j}^*$ }
				\State $\Gamma_{(i+1)j}^* = $ path\_temp
				\State $F_{(i+1)j}^* =$ cost\_temp
			\EndIf
		\EndFor
	\EndFor
\EndFor
\State $\beta^* = \Gamma_{(N+1)1}^*$
    \end{algorithmic}  
\end{algorithm}

\subsection{Solution to Optimal Tethered Multi-Goal Visiting}\label{subsection_TMV}

The solver for \textbf{Problem~\ref{prob:TMV}} is stated as follows: 

\textbf{Stage 1: Goal Configuration Pre-Collection.}
A GCP process is again required here for collecting the admissible configurations that visit goals $p_1, \cdots, p_N$. 
Let there be $n_1, \cdots, n_N$ possible configurations to visit each goal, denoted as 
\begin{equation}
\left\{
\begin{aligned}
&\{c_{1j_1} = (p_1, \alpha_{1j_1})\},\ j_1 = 1, \cdots, n_1\\
&\qquad \vdots\\
&\{c_{Nj_N} = (p_N, \alpha_{Nj_N})\},\ j_N = 1, \cdots, n_N
\end{aligned}
\right.
\end{equation}

\textbf{Stage 2: Optimal Reconfiguration Construction.}
Recall that we have formally written the $0$-th and $(N+1)$-th goals (and their configurations) as the base point and the trivial configuration in Eqn.~(\ref{eqn:0andn+1}). 
Arbitrarily choosing the $j$-th admissible configuration that visits $p_i$ and the $k$-th admissible configuration that visits $p_{i+1}$, by performing an UPS process, the optimal tethered reconfiguration motion is stored as $\Gamma_i(j, k)$, 
\begin{equation}
\begin{aligned}
\Gamma_i(j, k) \triangleq &\Gamma_{i(i+1)}(j, k) \triangleq \alpha^{\mathcal{C}}_{\alpha_{ij}^{-1}*\alpha_{(i+1)k}}\\
&j = 1, \cdots, n_i,\ k = 1, \cdots, n_{i+1}
\end{aligned}
\end{equation}
The length of paths in $\Gamma_{i}$ are stored in $F_{i}$: 
\begin{equation}
\begin{aligned}
F_i(j, k) \triangleq& F_{i(i+1)}(j, k) \triangleq g(\Gamma_{i(i+1)}(j, k))\\
&j = 1, \cdots, n_i,\ k = 1, \cdots, n_{i+1}
\end{aligned}
\end{equation}
Picking up one admissible configuration for every goal location, they uniquely specify a valid solution of TMV: 
\begin{equation}
\begin{aligned}
(j_1, \cdots, &j_{N})\\
\mapsto \alpha_0&*\Gamma_{0}(1, j_1)*\Gamma_{1}(j_1, j_2)*\cdots\\
&*\Gamma_{N-1}(j_{N-1}, j_N)*\Gamma_{N}(j_N, 1),\ \\
&j_i = 1, \cdots, n_i, i = 1, \cdots, N
\end{aligned}
\end{equation}
with the movement cost 
\begin{equation}
\sum\limits_{i = 0}^N F_i(j_i, j_{i+1})
\end{equation}
Here, note that the reconfiguration motions from or to the home configuration have been constructed by the GCP process, they need not be constructed again. 
Concretely, $\Gamma_0(1, j_1)$, the optimal reconfiguration motion from $(p_0, \alpha_0)$ to $(p_1, \alpha_{1j_1})$, is the path $S_{\alpha_{1j_1}}^\mathcal{C}$ which has been generated by the GCP process. 
Also, $\Gamma_N(j_N, 1)$ is the backtracking of $S^{\mathcal{C}}_{\alpha_{Nj_N}}$. 
Therefore, the range of $i$ that requires reconfiguration motion construction is from $1$ to $N-1$, i.e., the number of UPS processes is $\sum_{i=1}^{N-1}n_i\times n_{i+1}$. 
After constructing all optimal tethered reconfiguration motions, all valid TMV solutions could be constructed, whose number is $n_1\times n_2\times \cdots\times n_N$, wherein the optimal one is the least-cost one.
In practical, non-enumerative search of the optimal TMV solution is naturally motivated, presented in the next stage. 

\begin{figure*}
\centering
\includegraphics[width=0.98\textwidth]{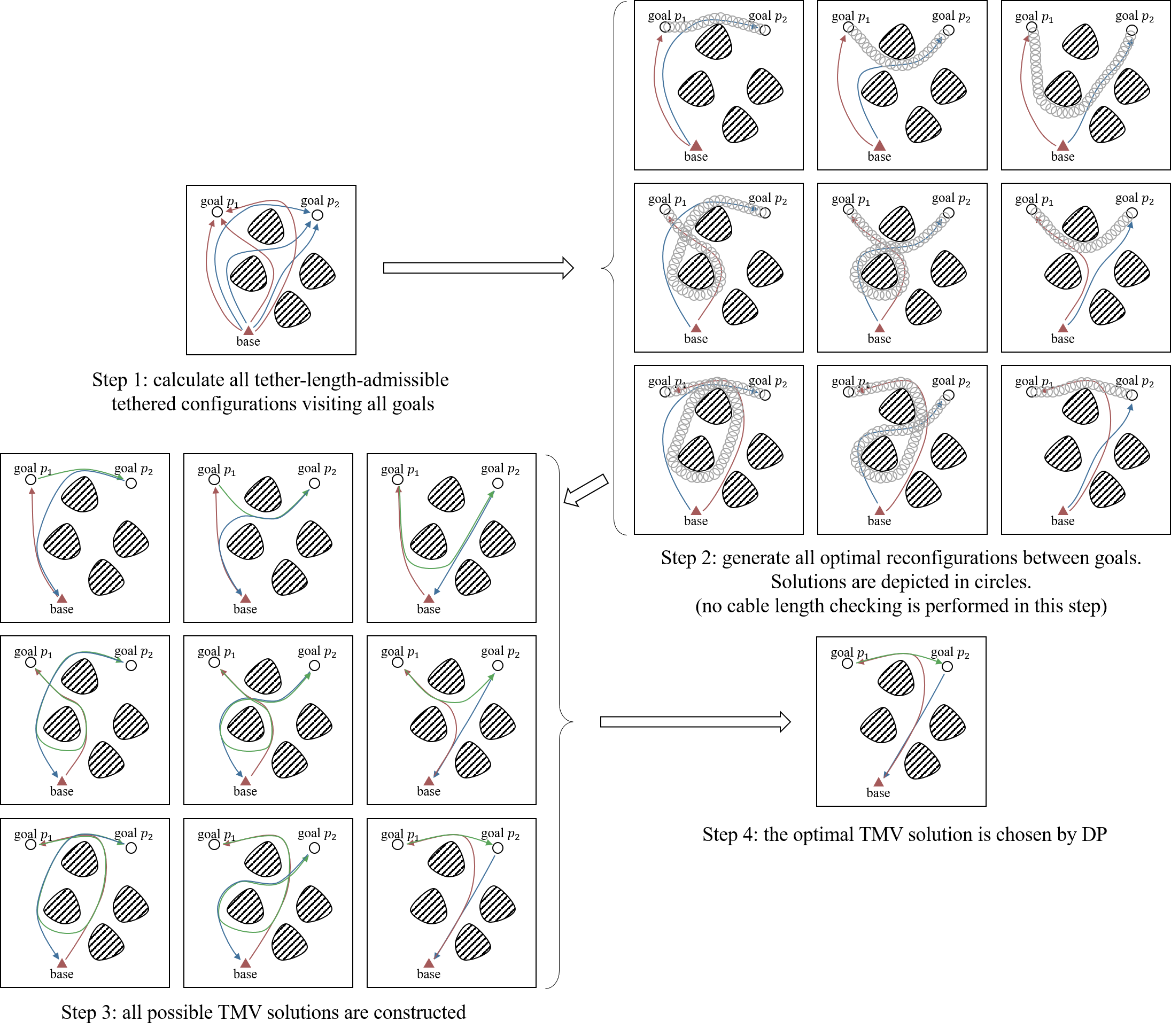}
\caption{Flowchart of the proposed optimal TMV solver for an optimal tethered $2$-goal visiting task. 
Three admissible configurations to visit each goal are shown in Step 1. 
Then, arbitrarily picking up one admissible configuration for each goal location, a UPS process is used to construct the optimal TR solution. 
All the $9$ optimal TR solutions are shown in Step 2. 
After that, we concatenate the path segments from $p_0$ to $p_1$ (constructed in Step 1), $p_1$ to $p_2$ (constructed in Step 2), and $p_2$ to $p_0$ (constructed in Step 1), shown in Step 3. 
Finally, the optimal TMV solution is the least-cost one among all the motions, shown in Step 4. }\label{fig:tmv_solution}
\end{figure*}

\textbf{Stage 3: Constructing the Optimal Solution.}
It is further observed that the construction of the optimal TMV solution satisfies the principle of optimality of dynamic programming (DP). 
Denote the optimal sub-TMV problem as the optimal TMV problem with a truncated sequence of goals. 
For any determined configuration $c_{ij_i}$ to visit $p_i$, the truncated sub-TMV solutions, visiting $(p_0, p_1, \cdots, p_{i-1}, p_i)$ in order and $(p_i, p_{i+1}, \cdots, p_N, p_0)$ in order, have well-defined initial configuration and final configuration. 
And the truncated part of the optimal TMV solution must also be optimal sub-TMV solutions. 
Hence a dynamic programming approach with $(N+1)$ stages is proposed: 
Denote $\Gamma^*_{i}(j_i)$ as the optimal sub-TMV solution for the robot transiting from $c_0$ to $c_{ij_i}$, and $F^*_i(j_i)$ be its length, then for $i = 0, \cdots, N$, 

\begin{equation}
\begin{aligned}
&F^*_{i+1}(j_{i+1}) = \min\limits_{1\leq j_i\leq n_i}\left\{F^*_i(j_i) + F_i(j_i, j_{i+1})\right\}\\
&j_i^* = \mathop{\rm argmin}\limits_{1\leq j_i \leq n_i}\{F^*_i(j_i) + F_i(j_i, j_{i+1})\}\\
&\Gamma^*_{i+1}(j_{i+1}) = \Gamma^*_{i}(j_i^*)*\Gamma_{i}(j_i^*, j_{i+1})\\
&\qquad\qquad\qquad 1\leq j_{i+1} \leq n_{i+1},\ i = 0, \cdots, N 
\end{aligned}
\end{equation}
The initial state is given as 
\begin{equation}
\Gamma^*_0(1) = c_0,\ F^*_0(1) = 0
\end{equation}
In the $i$-th stage, $\Gamma^*_{i+1}$ and $F^*_{i+1}$ are calculated, and $\Gamma^*_{N+1}(1)$ will be our desired optimal TMV solution, with $F^*_{N+1}(1)$ being its length. All the paths involved in the DP process have been constructed, hence the DP process is numerical. 

In summary, solving an optimal TMV problem requires a GCP process, $\sum_{i=1}^{N-1} n_i\times n_{i+1}$ UPS processes, and solving a numerical DP process. 
The algorithmic diagram for solving the optimal TMV problem is provided in \textbf{Algorithm~\ref{alg:TMV}}. 
An illustration of the proposed algorithm applied to an optimal tethered $2$-goal visiting task is depicted in Fig.~\ref{fig:tmv_solution}.

\subsection{Optimal Tethered Travelling Salesman Problem}
Finally, to find the optimal solution of  \textbf{Problem~\ref{prob:TTSP}}, we formulate the optimal TTSP problem into a generalised travelling salesman problem (GTSP). 
This is inspired by the fact that each goal location only needs to be visited once. 

\textbf{Stage 1: Configuration Collection.} This stage is the same as the \textbf{Stage 1} in Section~\ref{subsection_TMV}, so we omit it. 

\textbf{Stage 2: Optimal Reconfiguration Construction. } 
Noticing that optimally solving a TSP problem is inevitably enumerative, hence all possible motion segments (i.e., the optimal tethered motion between any two configurations that visit different goal locations) need to be constructed. 
The number of all goal configurations is $\sum_{i=0}^{N+1} n_i$, hence the reconfiguration motions (and their length) can be stored in a symmetric matrix $\Gamma$ (and $F$) whose dimension is $(\sum_{i=0}^{N+1} n_i)\times (\sum_{i=0}^{N+1} n_i)$. 
Adopting the previous notations, the optimal reconfiguration motion from the $k_i$-th configuration visiting goal $i$ to the $k_j$-th configuration visiting goal $j$ is indexed by  
\begin{equation}
\begin{aligned}
&\Gamma_{ij}(k_i, k_j) = \Gamma(\sum\limits_{l = 0}^{i-1}n_l + k_i, \sum\limits_{l = 0}^{j-1}n_l + k_j)\\
&F_{ij}(k_i, k_j) = F(\sum\limits_{l = 0}^{i-1}n_l + k_i, \sum\limits_{l = 0}^{j-1}n_l + k_j)\\
&\qquad\qquad\qquad\qquad\qquad i, j = 0, \cdots, N+1
\end{aligned}
\end{equation}
Since the robot need not visit the same goal location multiple times, the diagonal blocks of $\Gamma$ are empty, 
\begin{equation}
\begin{aligned}
&\Gamma_{ii}(k_{i1}, k_{i2}) = \Gamma(\sum\limits_{l = 0}^{i-1}n_l + k_{i1}, \sum\limits_{l = 0}^{i-1}n_l + k_{i2}) = \varnothing,\\
&\qquad\qquad\qquad\qquad k_{i1}, k_{i2} =1, \cdots, n_i,\ i = 0, \cdots, N+1
\end{aligned}
\end{equation}
And when either $i$ or $j$ is equivalent to $0$ or $N+1$, the paths are directly obtained from the GCP module. 
Hence the number of UPS processes is $(\sum_{i=1}^N n_i)^2 - \sum_{i=1}^N n_i^2$. 

\textbf{Stage 3: Constructing the Optimal Solution. }
After all reconfiguration motions have been collected, we regard each set of configurations $C_i$ as a cluster. 
Then the optimal TTSP problem is essentially to find a distance-optimal travelling tour that visits one configuration $c_i$ from each cluster $C_i$, which is a generalised travelling salesman problem (GTSP). 
All the optimal reconfiguration motions have been constructed, hence the remaining unsolved part is a numerical GTSP problem in Graph Theory~\cite{Lien1993Transformation}. 

In summary, the optimal TTSP problem is solved by decomposing it into a GCP process, $(\sum_{i=1}^N n_i)^2 - \sum_{i=1}^N n_i^2$ UPS process, and a numerical GTSP process.

\begin{figure}[t]
\centering
\subfloat[]{
\includegraphics[width = 0.44\textwidth]{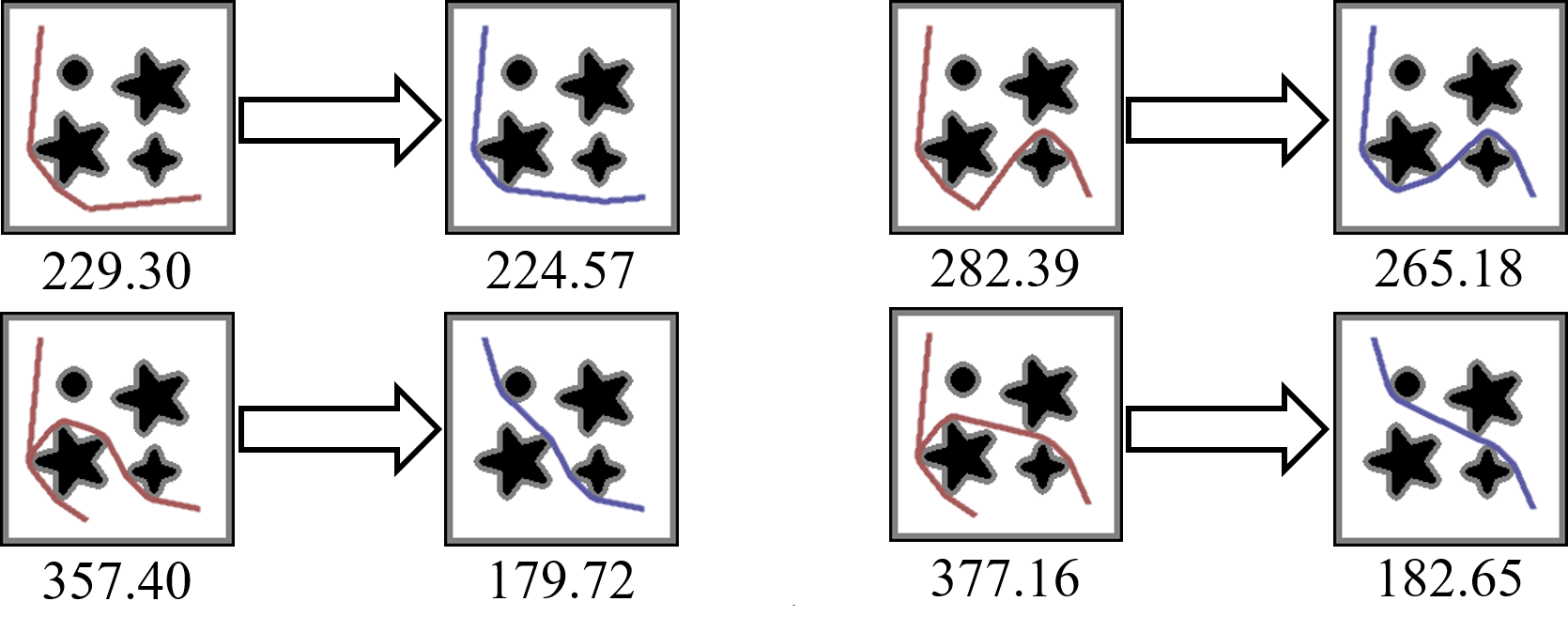}\label{fig:ups:a}
}\\
\subfloat[]{
\includegraphics[width = 0.48\textwidth]{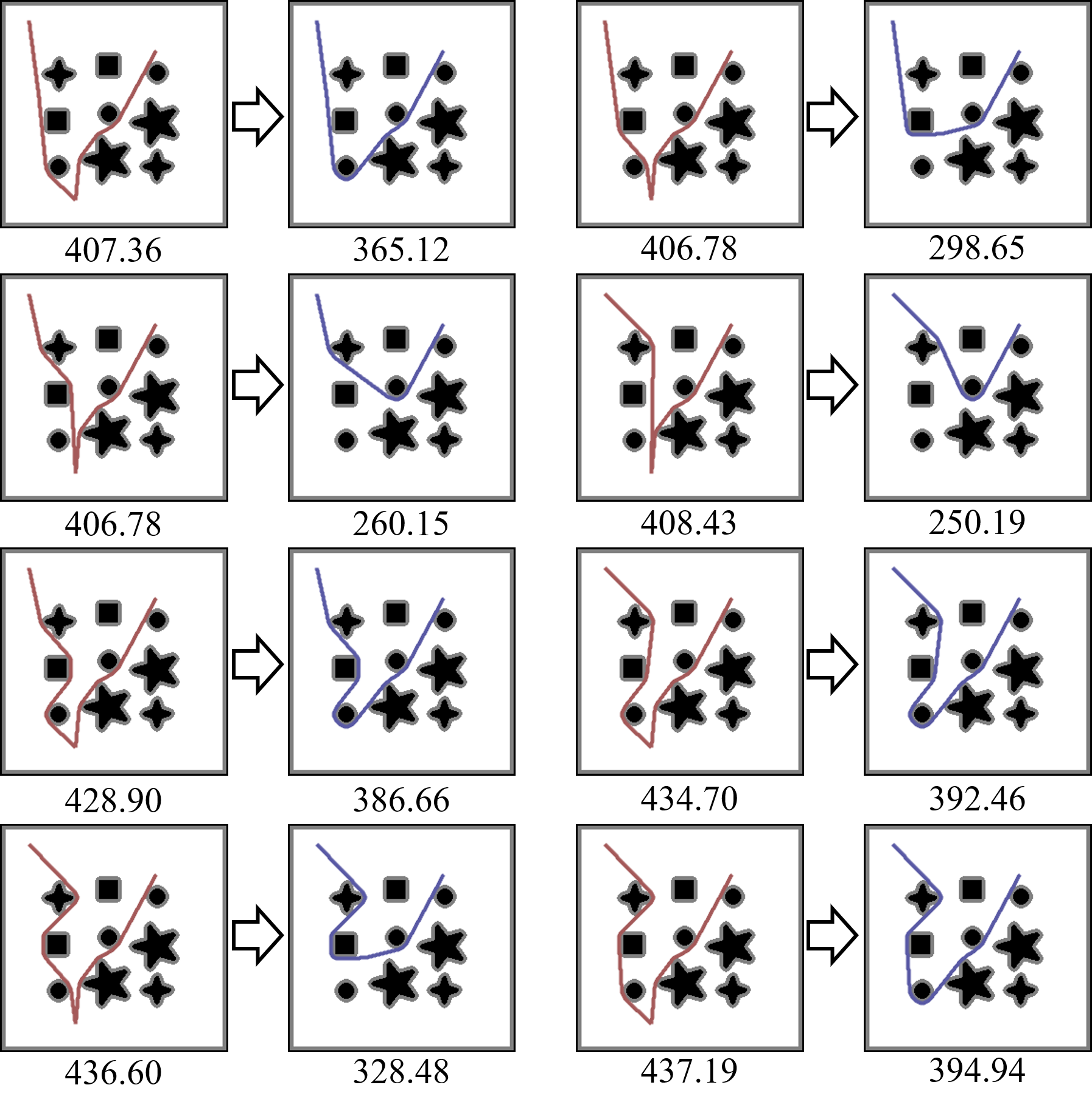}\label{fig:ups:b}
}
\caption{Illustrations of untethered path shortening processes. 
Red curves are the paths to be locally shortened, and blue paths are UPS resultant paths. 
The length of paths are listed for reference. 
(a) The four UPS processes performed when solving the optimal TP in the $160\times 160$ map. 
(b) The eight UPS processes performed when solving the optimal TP in the $240\times 240$ map. 
}\label{fig:ups}
\end{figure}

\begin{figure}[t]
\centering
\subfloat[Problem]{
\includegraphics[width=0.16\textwidth]{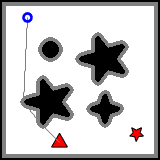}\label{fig:TP:a}
}
~~~~~~~
\subfloat[Solution]{
\includegraphics[width=0.16\textwidth]{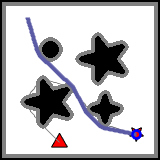}\label{fig:TP:b}
}\\
\subfloat[Problem]{
\includegraphics[width=0.23\textwidth]{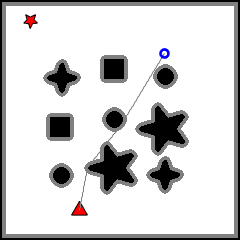}\label{fig:TP:c}
}
\subfloat[Solution]{
\includegraphics[width=0.23\textwidth]{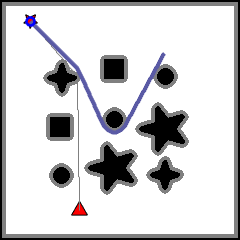}\label{fig:TP:d}
}
\caption{Illustration of the classic optimal tethered planning task, from a given starting configuration to a goal location. (a) The problem settings of the task in a map of size $160\times 160$ grids. 
The base point, the robot, the goal location, and the tether curve are marked in a red triangle, a blue circle, a red star, and a grey curve, respectively. 
(b) The optimal solution to the optimal TP problem in (a) is visualised as a blue curve. 
(c) The problem settings of the task in a map of size $240\times 240$ grids. The visualisations are the same as in (a). 
(d) The optimal solution to the optimal TP problem in (c) is visualised as a blue curve. 
}\label{fig:TP}
\end{figure}

\begin{figure*}
\centering
\subfloat[Goal $1$]{
\includegraphics[width = 0.15\textwidth]{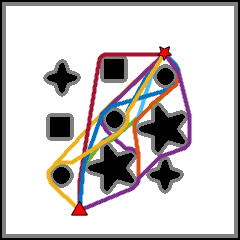}\label{fig:ksnpp:a}
}
\subfloat[Goal $2$]{
\includegraphics[width = 0.15\textwidth]{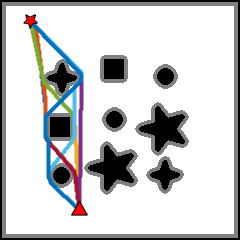}\label{fig:ksnpp:b}
}
\subfloat[Goal $3$]{
\includegraphics[width = 0.15\textwidth]{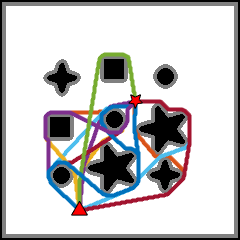}
}
\subfloat[Goal $4$]{
\includegraphics[width = 0.15\textwidth]{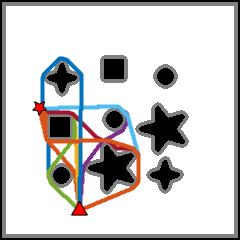}
}
\subfloat[Goal $5$]{
\includegraphics[width = 0.15\textwidth]{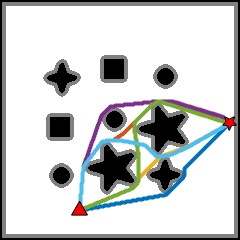}
}
\subfloat[Goal $6$]{
\includegraphics[width = 0.15\textwidth]{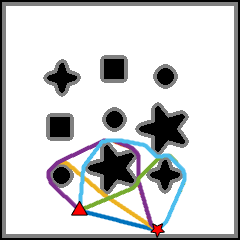}\label{fig:ksnpp:f}
}\\
\subfloat[Visit $1$]{
\includegraphics[width = 0.15\textwidth]{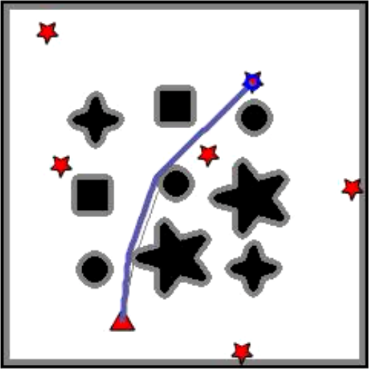}\label{fig:ksnpp:g}
}
\subfloat[Visit $2$]{
\includegraphics[width = 0.15\textwidth]{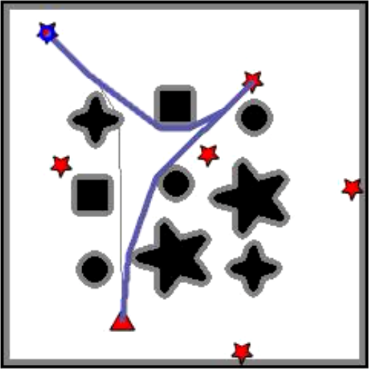}
}
\subfloat[Visit $3$]{
\includegraphics[width = 0.15\textwidth]{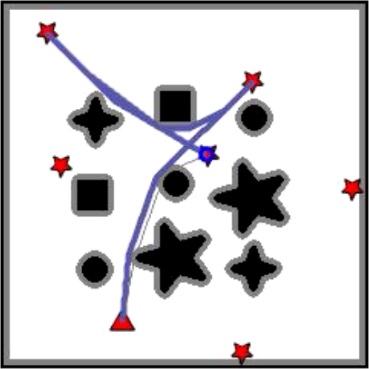}
}
\subfloat[Visit $4$]{
\includegraphics[width = 0.15\textwidth]{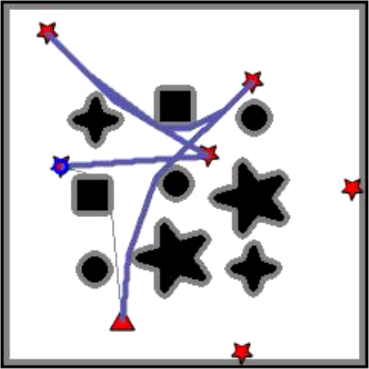}
}
\subfloat[Visit $5$]{
\includegraphics[width = 0.15\textwidth]{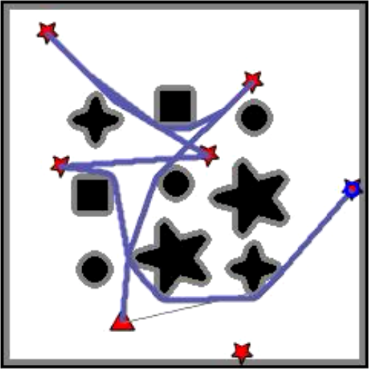}
}
\subfloat[Visit $6$]{
\includegraphics[width = 0.15\textwidth]{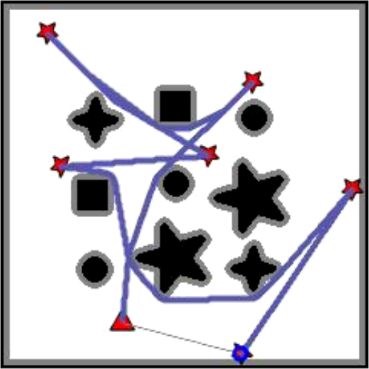}\label{fig:ksnpp:l}
}\\
\subfloat[The tether length whilst the robot motion. ]{
\includegraphics[width = 0.9\textwidth]{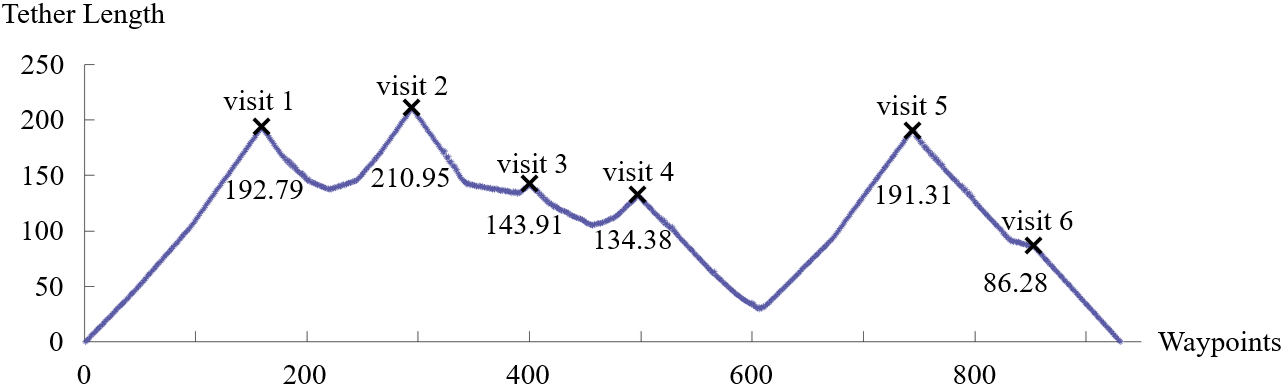}\label{fig:ksnpp:m}
}
\caption{
Illustration of the solution to an optimal tethered $6$-goal visiting task where the maximum tether length constraint is set as 250 (grids). 
The base point, the robot, the goal locations, and the tether curve are again marked in a red triangle, a blue circle, red stars, and a grey curve, respectively.
(a)$\sim$(f) The thick curves are the shortest non-homotopic paths from the base point to the goals, subject to the maximum tether length constraint. 
Different paths are shown in different colours. 
Figure (g)$\sim$(l) show the stills of the optimal TMV motion when the robot reaches each goal location. 
The path that the robot has visited is visualised in a thick blue curve. 
Figure (m) shows the tether length as a function curve with the robot moving along the optimal motion. 
}\label{fig:ksnpp}
\end{figure*}

\begin{figure*}[t]
\centering
\subfloat[Goal $1$]{
\includegraphics[width = 0.15\textwidth]{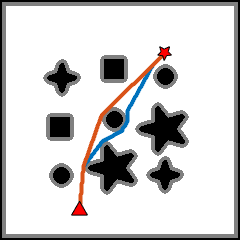}
}
\subfloat[Goal $3$]{
\includegraphics[width = 0.15\textwidth]{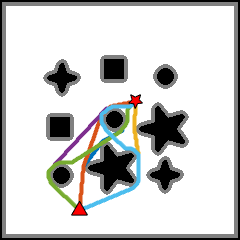}
}
\subfloat[Goal $4$]{
\includegraphics[width = 0.15\textwidth]{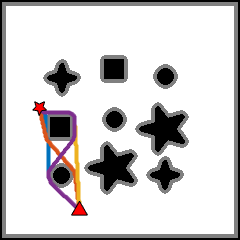}
}
\subfloat[Goal $5$]{
\includegraphics[width = 0.15\textwidth]{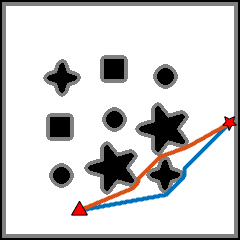}
}
\subfloat[Goal $6$]{
\includegraphics[width = 0.15\textwidth]{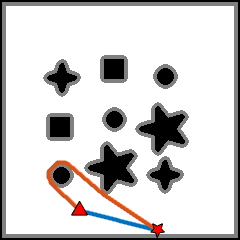}
}\\
\subfloat[Visit $1$]{
\includegraphics[width = 0.15\textwidth]{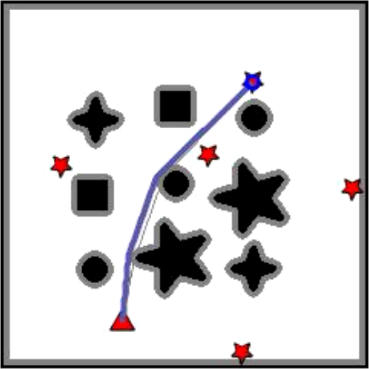}
}
\subfloat[Visit $3$]{
\includegraphics[width = 0.15\textwidth]{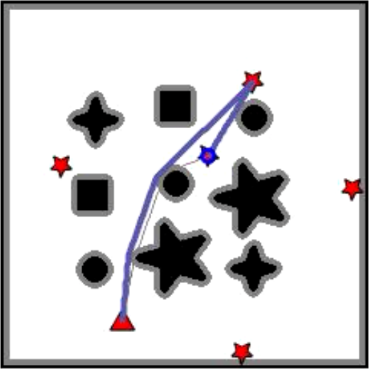}
}
\subfloat[Visit $4$]{
\includegraphics[width = 0.15\textwidth]{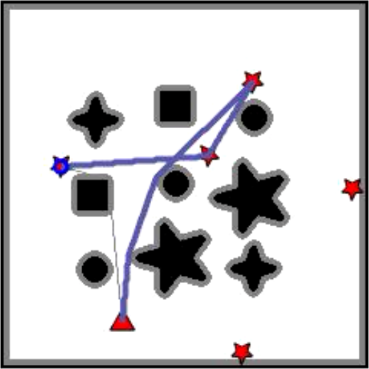}
}
\subfloat[Visit $5$]{
\includegraphics[width = 0.15\textwidth]{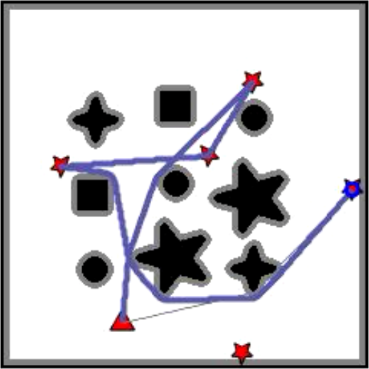}
}
\subfloat[Visit $6$]{
\includegraphics[width = 0.15\textwidth]{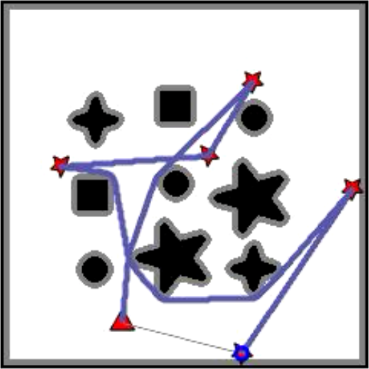}
}
\caption{Illustration of the paths when solving the same tethered $6$-goal visiting task under the maximum tether length constraint as $200$ grids. 
The base point, the robot, the goal location, and the tether curve are again marked in a red triangle, a blue circle, red stars, and a grey curve, respectively.
(a)$\sim$(e) The shortest non-homotopic paths to visit each goal location, where the top-left goal location $(31, 219)$ is not reachable. 
Different paths are visualised in different colours. 
Figure (f)$\sim$(j) show the stills of the optimal TMV motion when the robot reaches each goal location. 
The path that the robot has visited is visualised in a thick blue curve. 
}\label{fig:200}
\end{figure*}

\section{Experimental Results}\label{section_experiment}
The main result of this paper is the convexity of the workspace of a tethered robot moving in a planar environment. 
This is an intrinsic and fundamental property which eventually avoids the homotopy-aware path searching processes within the pre-calculated workspace. 
To the best of the author's knowledge, there did not exist a similar proposition before. 
As direct applications of the workspace convexity, the algorithms are improved in two aspects: 
The workspace pre-calculation (WP) process is replaced with a goal configuration pre-calculation (GCP) process, and the homotopy-aware path searching process is replaced with a locally untethered path shortening (UPS) process. 

Firstly, the details of the experiments are listed in Section~\ref{section_exp_detail}. 
Then, the efficiency of untethered path shortening processes is reported in Section~\ref{section_exp_ups}. 
After that, in the simulated experiments in Section~\ref{section_exp_TP}, the performance of the workspace convexity is evaluated by comparative experimentation of solving optimal TP tasks. 
Further, the case studies in Section~\ref{section_exp_TMV} show in detail how the optimal tethered multi-goal visiting (TMV) problem is solved within a reasonable computational time. 
Finally, Section~\ref{section_exp_TTSP} reports the solution of an optimal tethered travelling salesman problem. 
The reader is referred to the supplementary video for the animation of the simulated robot motion. 

\subsection{Implementation Details}\label{section_exp_detail}

The implementation details are provided as follows: 
\begin{enumerate}
\item The configurations with a self-crossing tether are ignored. 
\item We assume that the 2D C-space (the set of all collision-free locations) has been effectively pre-calculated, hence the collision-checking process is simply reading values from the C-space matrix. 
\item The UPS process is simply implemented as locally replacing a part of the curve with a homotopic Bresenham line segment in 2D grids~\cite{Bresenham1965Algorithm}.  
We loop through all parts of the curve to ensure the correctness of the resultant curve. 
\item Precise calculation of the tether length is not a focus of this paper, so we estimate the tether length of a configuration (i.e., the locally shortest curve in $\mathcal{M}_{\rm free}$) the same as the length of the collision-free locally shortest path homotopic to the tether (i.e., the locally shortest curve in $\mathcal{C}$). 
From a perspective of fair comparison this is doable: 
It only makes existing solvers faster (because the workspace is smaller) but does not speed up the proposed algorithm (because the proposed algorithm does not access all configurations in the workspace). 
\item The shortest non-homotopic path planning algorithm proposed in~\cite{Yang2022Efficient} is adopted as the GCP process. 
~\cite{Yang2022Efficient} aimed to find the $k$ collision-free shortest non-homotopic paths towards the goal. 
To collect all configurations, the branch pruning mechanism in~\cite{Yang2022Efficient} is not activated, which means that the node expansion will terminate only when the tree edge crosses its parent edges, or when the maximum tether length constraint is violated. 
Each output curve of~\cite{Yang2022Efficient} is a collision-free locally shortest curve, so symbolically it is $S^\mathcal{C}_\alpha$. 
Then we locally shorten it in $\mathcal{M}_{\rm free}$ so that the resultant curve is the tether state. 
If the tether length is admissible, then we find an admissible configuration as well as its homotopic collision-free locally shortest path. 
\footnote{Since ~\cite{Yang2022Efficient} is still under review, we temporarily publicise its implementation as well as our modification to verify the correctness of the experiments presented in this paper, at: \\
\url{https://github.com/ZJUTongYang/ksnpp}\\
\indent
We highlight that the convexity, a property of the tethered robot workspace, do not need an open-sourced code as a specialised implementation. 
Thus other GCP implementations and UPS implementations are also applicable. }

\end{enumerate}

\subsection{Efficiency of Untethered Path Shortening Process}\label{section_exp_ups}
The effectiveness of the proposed tethered robot planning solutions lies in the assertion that untethered path shortening (UPS) processes are sufficiently efficient to replace the workspace path searching processes. 
Obviously, the computational time of UPS depends on the length of the curve to be shortened. 
In this regard, we exemplify its performance by two case studies, as shown in Fig.~\ref{fig:ups:a} and Fig.~\ref{fig:ups:b}. 

In a simulated $160\times 160$ grid map, shown in Fig.~\ref{fig:ups:a}, the paths to be locally shortened are visualised in red curves, and results are depicted in blue curves. 
The average computational time for carrying out each UPS process is 0.786ms (3.144ms in total). 
In a larger $240\times 240$ grid map, shown in Fig.~\ref{fig:ups:b}, the average computational time is longer, 1.127ms (9.017ms for computing all of the eight UPS processes). 
It is natural that with the untethered paths being longer, the time for UPS also increases. 
However, it is still much more efficient than performing a homotopy-aware path searching algorithm, which is presented in the next subsection. 

\subsection{Comparison on Optimal TP Problem}\label{section_exp_TP}

In this experiment, we report the algorithmic improvement of the workspace convexity property for solving an optimal TP problem. 
The robot's radius is set as $4$ (grids), and the maximum tether length is set as $250$ (grids). 
Two simulated scenarios are utilised for evaluation. 
In the smaller environment, the map size is $160\times 160$ (grids) and four obstacles are presented. 
The base point (at $(60, 18)$), the starting configuration (located at $(28, 143)$), and the goal location (at $(137, 26)$) are shown in Fig.~\ref{fig:TP:a}. 
And in the larger environment, there are $9$ internal obstacles, where the map size is $240\times 240$ (grids). 
The base point (at $(80, 30)$), the starting configuration (located at $(165, 187)$), and the goal location (at $(31, 219)$) are shown in Fig.~\ref{fig:TP:c}. 
For succinctness, we expand our analysis only on the larger map. 

\textbf{Algorithm in~\cite{Kim2014Path}. }
The algorithm proposed in~\cite{Kim2014Path} consists of two parts, pre-calculating the workspace and performing an A*~\cite{hart1968formal} search from the starting configuration until the nearest configuration whose robot stays at the goal location. 
Although there are only $240\times 240 = 57600$ 2D locations in the map, the number of admissible configurations in the workspace is much larger, 396903, therefore pre-calculating the whole workspace is time-consuming (3656ms on average). 
After that, the time for A* running in the workspace is 911ms. 
The total computational time is 4567ms. 

\textbf{Algorithm in~\cite{Kim2014Path} with workspace convexity.}
Taking the convexity of the workspace into consideration, we obtain an improved solver of the optimal tethered planning problem, ``\cite{Kim2014Path}+ours". 
Since GCP is a sub-task of WP, the whole workspace having been pre-calculated means that all admissible goal configurations have been collected. 
Then by workspace convexity, the A* search in the workspace is simplified to several UPS processes: 
Eight admissible goal configurations have been observed in the workspace, shown in Fig.~\ref{fig:ksnpp:b}, so $8$ optimal tethered reconfiguration motions are obtained by running UPS processes for $8$ times, as enumeratively illustrated in Fig.~\ref{fig:ups:b}, whose computational time is 9.017ms in total. 
As a result, the computational time for solving the same optimal tethered path planning problem is reduced to 3665ms. 

\textbf{Algorithm in~\cite{Yang2022Efficient}. }
The GCP process can be implemented by using the algorithm proposed in~\cite{Yang2022Efficient}, which collects all goal configurations without enumerating the workspace. 
The computational time for GCP is 87ms, as shown in Fig.~\ref{fig:ksnpp:b}. 
However, it in its original form is not a tethered path planner without the workspace convexity. 

\textbf{Algorithm in~\cite{Yang2022Efficient} with workspace convexity. }
Because of the workspace convexity, after~\cite{Yang2022Efficient} collects the $8$ admissible goal configurations (87ms), we only need to perform eight UPS processes (9.017ms), visualised in Fig.~\ref{fig:ups:b}. 
Then the optimal TP solution could be safely selected as the minimum length one among the eight optimal reconfiguration motions. 
In summary, the optimal TP problem can be solved in 96ms. 
Relative statistics are collected in Table~\ref{table:optimal_tp}.

\begin{table*}[t]\scriptsize
\centering
\caption{Experimental Results of Solving Optimal TP Problems}\label{table:optimal_tp}
\begin{tabular}{c|c|c|c|c|c|c|c}
\hline
&  \multirow{2}{*}{Tether Length$^1$}& \multirow{2}{*}{Map Size} & \multirow{2}{*}{Num. of Obstacles} & \multirow{2}{*}{Num. of Configurations} &\multicolumn{3}{|c}{Time$^1$}\\
\cline{6-8}
 & & & & & Configuration Collection & Pathfinding & Total\\
\hline
\hline
~\cite{Bhattacharya2012Topological}  &\multirow{8}{*}{250}  & \multirow{4}{*}{$160\times 160$} & \multirow{4}{*}{4} & 134611 & (Whole Workspace) 788ms & (A*) 309ms & 1097ms\\
\cline{1-1}\cline{5-8}
~\cite{Bhattacharya2012Topological}+ours & & & & 134611 & (Whole Workspace) 788ms & (UPS$\times 4$) \textbf{3.144ms} & \textbf{791ms}\\
\cline{1-1}\cline{5-8}
~\cite{Yang2022Efficient} & & & & 4 & (SNPP) 13ms & \multicolumn{2}{c}{Not a Tethered Path Planner}\\
\cline{1-1}\cline{5-8}
~\cite{Yang2022Efficient}+ours & & & & 4 & (SNPP) 13ms & (UPS$\times 4$) \textbf{3.144ms} & \textbf{16ms}\\
\cline{1-1}\cline{3-8}
~\cite{Bhattacharya2012Topological}  & & \multirow{4}{*}{$240\times 240$} & \multirow{4}{*}{9} & 396903 & (Whole Workspace) 3656ms & (A*) 911ms & 4567ms\\
\cline{1-1}\cline{5-8}
~\cite{Bhattacharya2012Topological}+ours & & & & 396903 & (Whole Workspace) 3656ms & (UPS$\times 8$) \textbf{9.017ms} & \textbf{3665ms}\\
\cline{1-1}\cline{5-8}
~\cite{Yang2022Efficient} & & & & 8 & (SNPP) 87ms & \multicolumn{2}{c}{Not a Tethered Path Planner}\\
\cline{1-1}\cline{5-8}
~\cite{Yang2022Efficient}+ours & & & & 8 & (SNPP) 87ms & (UPS$\times 8$) \textbf{9.017ms} & \textbf{96ms}\\
\hline
\end{tabular}
\begin{tablenotes}
\item All results have been averaged over $50$ runs for a fair evaluation. 
\item $^1$ In all testings, the maximum tether length is set invariant as 250 grids. 
This reveals that the variation of the computational cost is purely caused by the increasingly complicated path topologies with the number of obstacles growing. 
\end{tablenotes}
\end{table*}

\begin{table*}[t]\scriptsize
\centering
\caption{Simulated Results of the Optimal TMV Solution Given Different Maximum Tether Length Constraints }\label{tab:different_tether_length}
\setlength\tabcolsep{5pt}
\begin{tabular}{c|c|c|c|c|c|c|c|c|c|c|c}
\hline
\multirow{3}{*}{Task} & \multirow{3}{*}{Tether Length} & \multicolumn{6}{c|}{Number of Admissible Configurations to Visit Each Goal} & \multirow{3}{*}{Number of Reconfig.} & \multicolumn{3}{c}{Computational Time}\\
\cline{3-8}\cline{10-12}
& & (165, 187) & (31, 219) & (136, 139) & (40, 132) & (230, 117) & (158, 10) &  &  \makecell[c]{Configuration\\Collection} & UPS & Total\\
\hline
\hline
\multirow{2}{*}{TMV}& 250 (grids) & 11 & 8 & 15 & 9 & 6 & 6 & 432 & 489ms & 526ms & 1015ms\\
\cline{2-12}
& 200 (grids) & 2 & 0 & 6 & 4 & 2 & 2 & 48 & 164ms & 43ms & 207ms\\
\hline
\multirow{2}{*}{TTSP} & 250 (grids) & 11 & 8 & 15 & 9 & 6 & 6 & 2462 & 489ms & 2997ms & 3486ms\\
\cline{2-12}
& 200 (grids) & 2 & 0 & 6 & 4 & 2 & 2 & 192 & 164ms & 172ms & 336ms\\
\hline
\end{tabular}
\begin{tablenotes}
\item All results have been averaged over $50$ runs for a fair evaluation. 
\end{tablenotes}
\end{table*}

\subsection{Illustration on Optimal TMV Problem}\label{section_exp_TMV}
In this subsection, an optimal tethered $6$-goal visiting task is solved to reveal the algorithmic efficiency by considering the workspace convexity. 
Relevant statistics are collected in Table.~\ref{tab:different_tether_length}. 
Given the same map as above (the $240\times 240$ one where we provide in-depth analysis), the same robot model whose maximum tether length is 250 (grids), and the same base point location as above, let there be a sequence of $6$ goals:
\begin{equation*}
(165, 187), (31, 219), (136, 139), (40, 132), (230, 117), (158, 10) 
\end{equation*}
In the first stage, we again utilise the SNPP method in~\cite{Yang2022Efficient} to collect all admissible goal configurations, visualised in Fig.~\ref{fig:ksnpp:a}$\sim$\ref{fig:ksnpp:f}. 
The algorithm in~\cite{Yang2022Efficient} is again adopted which solves the GCP task in 489ms. 
There are $11, 8, 15, 9, 6, 6$ admissible goal configurations that visit goal $1, \cdots, 6$. 
Then, in the second stage, the number of untethered path shortening processes is $11\times 8 + 8\times 15+\cdots + 6\times 6 = 432$, whose computational time is 526ms. 
Finally, the optimal configuration to visit each goal is selected by a numerical dynamic programming solver whose computational time is negligible. 
The optimal TMV solution is shown in Fig.~\ref{fig:ksnpp:g}$\sim$\ref{fig:ksnpp:l}. 
In summary, the total computational time for solving the optimal tethered $6$-goal visiting task is 1015ms.

The tether length whilst the robot is moving along the optimal tethered motion has also been recorded in Fig.~\ref{fig:ksnpp:m} to better appreciate the convexity of the workspace. 
The points marked by ``$\times$" indicate the configurations when the robot reaches a goal. 
As claimed in \textbf{Corollary~\ref{cor:max_length_point}}, for any optimal reconfiguration motion (i.e., for any interval of the function curve without a ``$\times$" being enclosed), the configuration that has the longest tether must be one of its endpoint configurations. 
In particular, for each pair of consecutive ``$\times$"s, one of them acts as the maximum value of the tether length function in the interim interval. 

The algorithmic efficiency is also evaluated with a different maximum tether length constraint. 
In this case study, the maximum tether length is reduced from $250$ (grids) to $200$ (grids).
As one can expect, there will be a less number of admissible goal configurations, being constructed in a shorter computational time, and some goal locations might become unreachable. 
Relevant statistics are collected in Table~\ref{tab:different_tether_length}. 
The time for GCP computing is reduced to 164ms. 
Results show that the number of possible configurations to visit each goal is reduced to $2, 0, 6, 4, 2, 2$ (the second goal is not reachable).  
Then $2\times 6 + 6\times 4 + \cdots + 2\times 2 = 48$ UPS processes are performed, which takes 43ms. 
Hence the optimal TMV task with maximum tether length constraint $200$ (grids) can be solved in 207ms. 
The animation of the robot motions are visualised in the supplementary video. 

\begin{figure}[t]
\centering
\subfloat[]{
\includegraphics[width = 0.22\textwidth]{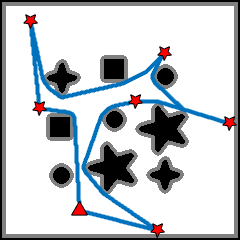}\label{fig:TTSP:a}
}
\subfloat[]{
\includegraphics[width = 0.22\textwidth]{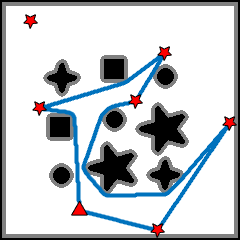}\label{fig:TTSP:b}
}
\caption{Optimal tethered travelling salesman motions. 
(a) The solution when the maximum tether length is 250 (grids). 
(b) The solution when the maximum tether length is 200 (grids). 
In this case, the left-top location $(31, 219)$ is not reachable. 
}\label{fig:TTSP}
\end{figure}

\subsection{Illustration of the Optimal TTSP Solution}\label{section_exp_TTSP}
Given the same environment settings as above, the optimal tethered travelling salesman problem (TTSP) solutions are illustrated in Fig.~\ref{fig:TTSP}. 
Let the maximum tether length constraint be 250 grids. 
The first step, solving GCP, is exactly the same as what we did in the optimal TP experiments, hence is omitted here. 
Then, in the second step, $(11+\cdots + 6)^2 - 11^2-\cdots - 6^2 = 2462$ UPS processes are required, which take 2997ms in total. 
In summary, the proposed solver spends 3486ms in transforming the optimal TTSP task into a numerical GTSP task. 
Testing the algorithm with the maximum allowable tether length being set as 200 grids, the time for solving GCP is 164ms, and only 192 UPS processes are required which take 172ms, which means that the transformation from optimal TTSP to numerical GTSP is finished in 336ms. 
Relevant statistics are collected in Table~\ref{tab:different_tether_length}. 

It can be noticed that the visiting order of goals has changed under different tether length constraint settings. 
When the maximum tether length is 250 (grids), the robot can efficiently visit goal 1($165, 187$), 5($(230, 117)$), and 3($(136, 139)$) in order, as shown in Fig.~\ref{fig:TTSP:a}. 
However, with the maximum tether length being reduced to 200, the configuration that visits goal 5 violates the maximum tether length constraint. 
In this case, it would be more efficient to change the goal visiting order to 1, 3, and finally 5, as shown in Fig.~\ref{fig:TTSP:b}. 
The animation of the TTSP solutions are also illustrated in the supplementary video. 

\section{Conclusion}\label{section_conclusion}

This paper focused on the distance-optimal path planning problems for a tethered omni-directional robot in a 2D environment. 
The tethered robots have been verified to have appealing advantages in maintaining stable energy supply and communication links. 
The state of robot tether imposes complicated topological difficulties, and the maximum cable length constraint makes the planning tasks for tethered mobile robots further non-trivial and non-effectively solved. 

The main contribution of this paper is the proof of the convexity of the tethered robot workspace, as well as a set of distance-optimal tethered path planning algorithms that leverage the workspace convexity. 
It has been proven that the optimal reconfiguration motion between two tether-length-admissible configurations must be the untethered locally shortest collision-free path whose homotopy is specified by the concatenation of the tether curve of the endpoint configurations. 
In this respect, the distance-optimal transition between two configurations can be simply obtained by an untethered path shortening (UPS) process in the 2D environment, instead of a pathfinding process in the workspace. 
Regarding the optimal reconfiguration motion as a building block, high-level distance-optimal tethered planning tasks, such as the classic optimal tethered path planning task, optimal tethered multi-goal visiting task, and the tethered travelling salesman problem, have been significantly simplified. 
In particular, the most time-consuming module in previous algorithms, the workspace pre-calculation process, has been simplified to the goal configuration pre-calculation process which could be effectively executed. 
Extensive comparisons and case studies have been provided to show the merit of the proposed theory and prove the validity of the proposed algorithms in producing highly effective optimal tethered robot motions, supplemented by a detailed video. 


\bibliographystyle{ieeetr} 
\bibliography{TRO21_optimalTMV_v2}

\newpage

\vfill

\end{document}